\documentclass[11pt]{airesearch}
\usepackage[T1]{fontenc}
\usepackage{XCharter}

\usepackage{algorithm}
\usepackage{algpseudocode}

\ifdefined\final
  \usepackage[disable]{todonotes}
\else
  \usepackage[textsize=tiny]{todonotes}
\fi

\newcommand{\R}{\mathbb{R}}

\newcommand{\softmax}{\mathrm{softmax}}
\newcommand{\logsumexp}{\mathrm{logsumexp}}
\newcommand{\lse}{\logsumexp}
\providecommand{\argmax}{\operatorname*{arg\,max}}

\newcommand{\gumbel}{\mathrm{Gumbel}}
\newcommand{\cat}{\mathrm{Cat}}
\newcommand{\Unif}{\mathrm{Unif}}

\newcommand{\vct}[1]{\bm{#1}}
\newcommand{\mat}[1]{\bm{#1}}

\newcommand{\rev}[1]{#1}

\title{\LARGE \bfseries FlashSampling: Fast and Memory-Efficient \\Exact Sampling}
\author{
    \textbf{Tomas Ruiz}$^{1}$\thanks{Equal contribution; $^{\dagger}$Corresponding authors.}~~~~\textbf{Zhen Qin}$^{3}$\footnotemark[1]~~~~\textbf{Yifan Zhang}$^{2}$\footnotemark[2]~~~~\textbf{Xuyang Shen}$^{3}$\\[0.5mm]
    \textbf{Yiran Zhong}$^{3}$~~~~\textbf{Mengdi Wang}$^{2}$$^{\dagger}$\\[1.5mm]
    $^1$LMU Munich~~~~$^2$Princeton University~~~~$^3$FlashSampling
}

\date{}

\begin{document}
\maketitle

\begin{abstract}
Sampling from a categorical distribution is mathematically simple, but in large-vocabulary decoding, it often triggers extra memory traffic and extra kernels after the LM head. We present \textbf{FlashSampling}, an exact sampling primitive that fuses sampling into the LM-head matmul and never materializes the logits tensor in HBM. The method is simple: compute logits tile-by-tile on chip, add Gumbel noise, keep only one maximizer per row and per vocabulary tile, and finish with a small reduction over tiles. In tensor-parallel decoding, FlashSampling replaces the all-gather of logits with streaming peer-to-peer writes: This overlaps GPU-to-GPU communication with computation and HBM loads across up to 8 GPUs, with near-ideal scaling at large batch sizes. Our kernel is exact because $\argmax$ decomposes over partitions; grouped variants for online and tensor-parallel settings are exact by hierarchical factorization of the categorical distribution. FlashSampling demonstrates kernel-level speedups on decode workloads across 4 different datacenter GPUs (H100, H200, B200, B300), and in end-to-end vLLM experiments, it reduces time per output token by up to $10\%$ on the models we test. These results show that exact sampling, with no approximation, can be integrated into the matmul itself, consolidating the bandwidth-bound sampling step in an efficient epilogue.
\end{abstract}

\projectpage{https://github.com/FlashSampling/FlashSampling}

\begin{figure}[ht!]
\centering
\includegraphics[width=1.0\columnwidth]
{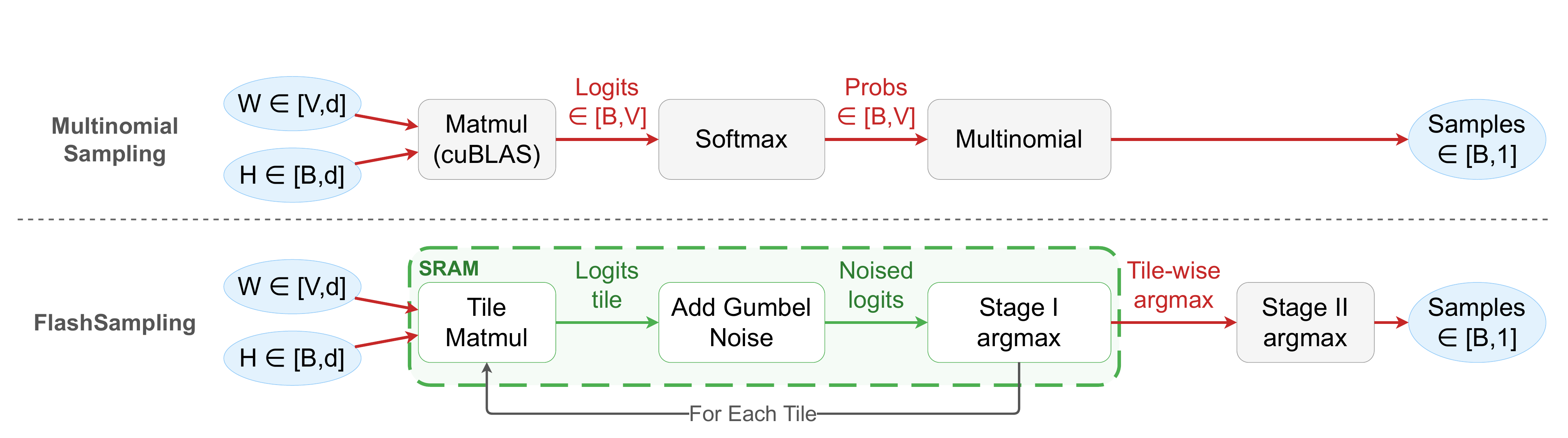}
\vspace{-2ex}
\caption{Conventional multinomial sampling (left) materializes the full $[B,V]$ logits tensor in HBM between the matmul and the sampler. FlashSampling (right) fuses sampling into the matmul epilogue, followed by a lightweight reduction over vocabulary tiles. Logits are computed tile-by-tile in on-chip memory, perturbed with Gumbel noise, and reduced without ever writing the full logits tensor to HBM. Red arrows denote HBM traffic; green arrows denote on-chip data movement.}
\label{fig:architecture}
\end{figure}

\section{Introduction}

Sampling from a categorical distribution is a small mathematical operation, but in large-categorical systems, it can become an expensive inner-loop primitive. Modern LLM serving stacks invoke sampling repeatedly during autoregressive decoding, on vocabularies with tens or hundreds of thousands of categories \citep{kwon2023efficient,ye2025flashinfer,maddison2014astar,huijben2022review}.
Recent measurements confirm the cost: sampling can account for over 10\% of token generation time even on a single GPU \citep{key2024approximate}, and 20--38\% in \rev{the} tensor-parallel settings \rev{of \citet{zhao2025simpledisaggregatingsamplinggpu}}, where logits must be gathered across ranks.
The bottleneck is usually not arithmetic, but the chain of separate sampling kernels that materialize, normalize, and scan the logits tensor.

At decode time, the LM-head projection already streams a large $[V,D]$ weight matrix from HBM ($V$=vocabulary size, $D$=hidden dimension). When the active batch is small, this projection is memory-bandwidth bound. 
Materializing the resulting $[B,V]$ logits tensor ($B$=batch size), launching extra kernels to normalize and sample from it, adds extra memory traffic but no useful model computation,
since the logits are immediately discarded after a single sample is drawn.
In this regime, the separate sampler is pure overhead \citep{dao2022flashattention, wijmans2025cut}.
Furthermore, in the multi-GPU tensor-parallel setting, decoding must additionally synchronize and communicate the logits across ranks: Sampling becomes a memory and communication problem when full logits are materialized. Exact sampling is often described as ``compute softmax, then sample'', which obscures the fact that exact sampling does not require forming probabilities at all.  

In this work, we introduce FlashSampling, which computes logits tile-by-tile on chip and writes only one candidate per row and per vocabulary tile, followed by a lightweight reduction. Exact sampling needs only the index of the largest perturbed logit, so there is no need to form a softmax, a prefix sum, or normalized probabilities; the method is exact and introduces no approximation. A simple hierarchical factorization yields exact online and distributed variants that keep only small summaries in flight and communicate only small summaries across ranks.

Our contributions can be summarized as follows:
\begin{enumerate}[leftmargin=*, itemsep=1pt]
  \item We introduce a two-stage design that computes logits and samples tile-by-tile within the LM-head epilogue, instead of materializing the full logits tensor to HBM.
  In multi-GPU decoding, FlashSampling overlaps cross-GPU communication with the matmul compute and HBM loads and scales near-ideally on up to 8 GPUs.
  \item We separate the two ingredients used in the paper: the fused tiled kernel is exact pathwise by $\argmax$ decomposition over vocabulary tiles \rev{and ranks}, while \rev{the} grouped \rev{and} online \rev{extensions} are exact in distribution by hierarchical factorization through group log-masses.
  \item We demonstrate consistent speedups in the decode regime across four different NVIDIA datacenter GPUs in microbenchmarks and in end-to-end vLLM evaluations. We provide a simple I/O cost model to predict speedups, show that real speedups exceed these predictions, and explain why.
\end{enumerate}

\section{Background}

\paragraph{Notation.}
Let $[V]:=\{1,\dots,V\}$. Let $\tilde{\vct{\ell}}\in(\R\cup\{-\infty\})^V$ denote \emph{transformed logits} after any deterministic operations such as additive bias, temperature scaling, or masking. We assume that each row has at least one finite entry; otherwise, the target categorical distribution is undefined. The target distribution is $p(i)=\frac{\exp(\tilde{\ell}_i)}{\sum_{j=1}^V \exp(\tilde{\ell}_j)}$.
Raw logits $\vct{\ell}$ are the special case $\tilde{\vct{\ell}}=\vct{\ell}$. We denote i.i.d.\ standard Gumbel variables by $g_i\sim \gumbel(0,1)$. 

\subsection{Why Sampling Is Expensive at Scale}

A common LLM sampling pipeline first computes logits with a GEMM (LM-head projection), and then transforms them, normalizes them to probabilities, and finally samples from them. An example is softmax followed by inverse-CDF sampling. 
Algorithm~\ref{alg:standard_sampling} in the appendix summarizes this pattern.
An alternative to softmax-based sampling is described below in Section~\ref{sec:gumbel-max-trick}.
Nevertheless, all samplers that materialize the logits to HBM must pay the price of at least one logits write, at least one logits re-read, and the price of the sampling work.
This adds avoidable HBM round-trips in the critical path of the memory-bound decode loop.
\[
\text{GEMM }(\text{produce logits})
\;\to\;
\text{write logits to HBM}
\;\to\;
\text{read logits for sampling}.
\]

\paragraph{Distributed logits.}
In the multi-GPU setting, the LM-head weights $\mat{W} \in \R^{V \times D}$ are sharded across $n$ GPUs along the vocabulary dimension $V$ (\emph{column-parallel} sharding of $\mat{W}^\top$~\citep{shoeybi2020megatronlmtrainingmultibillionparameter}).
As a result, each GPU computes only a shard of the full logits $\mat{Y}$.
Before sampling, all GPUs must communicate their logit shards via an all-gather collective, incurring per-GPU communication of order $B \cdot V$ and aggregate costs proportional to the number of GPUs $n$.
This overhead is most pronounced in deployments using high tensor-parallelism, often aimed at minimizing latency.

\subsection{The Gumbel-Max Trick}
\label{sec:gumbel-max-trick}
The classical Gumbel-Max trick states that exact categorical sampling can be performed by adding i.i.d.\ Gumbel noise and taking an $\argmax$:
\begin{theorem}[Gumbel-Max]
Given transformed logits $\tilde{\vct{\ell}}\in(\R\cup\{-\infty\})^V$, exact sampling from $\cat(\softmax(\tilde{\vct{\ell}}))$ is:
$i^\star = \argmax_{i\in[V]} \left(\tilde{\ell}_i + g_i\right),
g_i\sim\gumbel(0,1)\text{ i.i.d.}$

\end{theorem}
\noindent This classical result goes back to \citet{gumbel1954statistical} and is widely used in machine learning \citep{maddison2014astar,huijben2022review}.
The trick extends to sampling \emph{without} replacement via the Gumbel-Top-$k$ method \citep{kool2019stochastic}, though not to the common sampling \emph{with} replacement in LLMs.
The key point for this paper is simple: \emph{exact sampling does not require an explicit softmax}. It only requires the index of the largest perturbed logit.

\section{FlashSampling}

We now describe FlashSampling from simplest to most practical form.
The core algorithm is intentionally simple and introduces no approximation: maintain the largest perturbed score seen so far and its index.

\subsection{Exact Sampling via Online Gumbel-Max}


\paragraph{Algorithm.}
Generate i.i.d.\ Gumbels, compute $s_i=\tilde{\ell}_i+g_i$, and return $i^\star=\argmax_i s_i$. The computation can be performed online in a single pass that maintains only the current best score and its index, analogous to the online normalizer calculation for softmax \citep{milakov2018online}. No softmax, no normalization constant, and no prefix sum are required (see Algorithm~\ref{alg:streaming} in the Appendix).

\paragraph{Systems implication.}
Sampling reduces to a single reduction over perturbed logits. This naturally fits GPU reductions and removes the extra normalization and prefix-sum work used by common softmax-based pipelines.

\paragraph{Simplicity.}
The online algorithm keeps only two running state variables per row: the current best perturbed score and the corresponding index. This simplicity is what makes fusion with the LM-head epilogue practical.

\paragraph{GPU parallelization.}
Each threadblock can process one contiguous vocabulary chunk, or \emph{vocabulary tile}. The block computes perturbed scores for that chunk, keeps only the tile-local maximizer, and a small second-stage reduction selects the global maximizer across vocabulary tiles.

\subsection{FlashSampling Algorithm}
\label{sec:fused_mm_sample}

We now consider the common case where logits are produced by GEMM $
\mat{Y} = \mat{H}\mat{W}^\top \in \R^{B\times V}$, where $\mat{H}\in\R^{B\times D}$ are hidden states and $\mat{W}\in\R^{V\times D}$ are LM-head weights. We wish to sample one index per row from $\cat(\softmax(\mat{Y}_{b,:}))$, possibly after deterministic transforms such as temperature scaling, additive bias, or masking.

\paragraph{Goal: avoid materializing $\mat{Y}$.}
FlashSampling performs sampling inside the matmul kernel and writes only one candidate per row and per vocabulary tile, never the full $[B,V]$ logits tensor:
\begin{itemize}
  \item \textbf{Stage 1 (Fused Kernel):} compute one batch tile and one vocabulary tile on chip, apply deterministic transforms, add Gumbel noise, and keep the tile-local maximizer for each row.
  \item \textbf{Stage 2 (Reduction):} reduce over vocabulary-tile candidates to obtain one global sample per row.
\end{itemize}

\begin{algorithm}[ht!]
\caption{FlashSampling (two-stages)}
\label{alg:fused}
\begin{algorithmic}[1]

\Statex \textbf{Input:} Hidden states $\mat{H}\in\R^{B\times D}$
\Statex \hphantom{\textbf{Input:} }LM-head weights $\mat{W}\in\R^{V\times D}$
\Statex \hphantom{\textbf{Input:} }temperature $\tau>0$, optional mask/bias, RNG key

\Statex \textbf{Output:} Samples $\vct{i}^\star\in\{1,\dots,V\}^B$
\Statex

\Statex \textbf{Stage 1 (Fused Kernel):} for each batch tile $\mathcal{B}$ and vocabulary tile $\mathcal{V}_t$ in parallel
\State $\mat{Y}^{(t)}_{b,i} \gets \sum_{d=1}^{D} \mat{H}_{b,d}\, \mat{W}_{i,d}$ for $(b,i)\in\mathcal{B}\times\mathcal{V}_t$ \Comment{tiled matmul over $D$, kept on chip}
\For{each output element $(b,i)\in\mathcal{B}\times\mathcal{V}_t$}
  \State $\tilde{y}_{b,i} \gets \mathrm{transform}\!\left(\mat{Y}^{(t)}_{b,i} \right)$ \Comment{apply temperature, bias, mask}
  \State Draw $u_{b,i}\in(0,1)$ and set $g_{b,i}\gets -\log\!\big(-\log u_{b,i}\big)$
  \State $s_{b,i}\gets \tilde{y}_{b,i}+g_{b,i}$
\EndFor
\For{each row $b\in\mathcal{B}$}
  \State $(m_b^{(t)},\mathrm{idx}_b^{(t)}) \gets \max_{i\in\mathcal{V}_t} s_{b,i}$ \Comment{idx = global vocabulary index}
  \State Write $(m_b^{(t)},\mathrm{idx}_b^{(t)})$ to HBM
  \If{multi-GPU}
    \State Fan-out $(m_b^{(t)},\mathrm{idx}_b^{(t)})$ to peer ranks via P2P
  \EndIf
\EndFor
\If{multi-GPU}
  \State Cross-rank barrier \Comment{P2P writes are not collectives; sync before Stage 2}
\EndIf
\Statex \textbf{Stage 2 (Reduction):}
\For{each row $b\in\{1,\dots,B\}$}
  \State $t^\star \gets \argmax_t m_b^{(t)}$
  \State $i_b^\star \gets \mathrm{idx}_b^{(t^\star)}$
\EndFor
\State \Return $\vct{i}^\star$
\end{algorithmic}
\end{algorithm}

\paragraph{Why the two-stage design is simple.}
The fused stage does all the expensive work in the matmul epilogue. The second stage is only an $\argmax$ over a small candidate buffer of shape roughly $[B,\#\text{vocab tiles}]$. This design is easy to implement and already captures most of the benefit in the decode regime.

\paragraph{Why this avoids softmax.}
The algorithm never forms probabilities and never computes an explicit softmax. Exactness follows because it computes the same maximizer of the perturbed logits that a full Gumbel-Max pass would compute. For further discussion, please refer to~\ref{sec:details of flashsampling}.

\paragraph{Multi-GPU Communication.}
A naive sampler must wait until the end of the GEMM to assemble the full logits $\mat{Y}$ using an all-gather collective.
FlashSampling instead issues per-tile point-to-point writes (\emph{fan-out} pattern) from inside the matmul epilogue.
This broadcasts each tile-local candidate to the other GPUs via P2P (NVLink) as it is produced.
Because this is not a collective, an explicit cross-rank barrier follows the kernel before Stage 2.
The benefit is that GPU-to-GPU communication overlaps with the matmul's compute and HBM loads, instead of being deferred until after the GEMM.

The complete algorithm is provided in Algorithm~\ref{alg:fused}, and a more detailed analysis and proof of its correctness can be found in Sections~\ref{sec:details of flashsampling}, ~\ref{sec:theory}.

\subsection{IO Cost Model and Predicted Speedup}
\label{sec:cost_model}
We outline a simple model to reason about speedups. It consists of comparing the theoretical minimal data movement of FlashSampling to the baseline, and inferring speedups from the ratio.

\paragraph{Baseline Cost:}
The baseline runs two steps: (1) computing logits with a GEMM, and (2) sample from the logits.
The GEMM must read the weights $\mat{W}$ and hidden states $\mat{H}$ and write the logits $\mat{Y}$ once. An (idealized) sampling step reads the logits $\mat{Y}$ once, and writes the sampled index $i^\star$ once.
In practice, the baseline uses multiple kernels for sampling, not one.
The total data movement $M_{\mathrm{baseline}}$ from HBM is therefore:
\[
M_{\mathrm{baseline}} = \overbrace{\underbrace{VD}_{\text{read } \mat{W}} + \underbrace{DB}_{\text{read } \mat{H}} + \underbrace{VB}_{\text{write } \mat{Y}}}^{\text{GEMM}} + \overbrace{\underbrace{VB}_{\text{read } \mat{Y}} + \underbrace{B}_{\text{write } i^\star}}^{\text{sampling}}
\]

\paragraph{FlashSampling Cost:}
Sampling is fused into the GEMM epilogue, so the logits write and read are eliminated.
By avoiding the $\mat{Y}$ round-trip, fusion reduces data movement and improves arithmetic intensity.
\[
M_{\mathrm{fused}} = \overbrace{\underbrace{VD}_{\text{read } \mat{W}} + \underbrace{DB}_{\text{read } \mat{H}} + \underbrace{B}_{\text{write } i^\star}}^{\text{fused GEMM + sampling}}
\]

\paragraph{Predicted speedup:}
The IO-model speedup is the ratio of the two costs $M_{\mathrm{baseline}}$ and $M_{\mathrm{fused}}$.
Since both $1/V \approx 0$ and $D/V \approx 0$ for current LLMs, the speedup equation simplifies to:
\[
\frac{M_{\mathrm{baseline}}}{M_{\mathrm{fused}}}
= \frac{VD + DB + 2VB + B}{VD + DB + B}
= 1 + \frac{2}{D/B + D/V + 1/V}
\approx\; 1 + \frac{2B}{D}.
\]

The predicted speedup increases with batch size $B$ and decreases with model hidden size $D$, i.e.\ smaller models experience larger speedups.
We verify the IO model empirically by running FlashSampling with an optional flag that stores the computed logits back to HBM, isolating the logits data movement overhead with no other changes to the kernel.
The measured overhead tracks the predicted $2B/D$ ratio (Appendix~\ref{app:logits_ablation}, Table~\ref{tab:logits_ablation}).
Despite the IO model being accurate, FlashSampling achieves much larger speedups over the baseline than the IO model predicts, which we review in Section~\ref{sec:experiments}.

\section{Experiments}
\label{sec:experiments}

We evaluate FlashSampling at two levels: kernel-level microbenchmarks that isolate fused matmul-plus-sample across four GPU architectures, and end-to-end vLLM~\citep{kwon2023efficient} integration that measures autoregressive decode latency.
We implement FlashSampling using Triton~\citep{tillet2019triton}; the implementation source code and experiment scripts are provided in the anonymized supplementary zip archive.

\subsection{Setup}
\label{sec:setup}

\paragraph{Hardware.}
Kernel microbenchmarks are run on four NVIDIA datacenter GPUs spanning two architecture generations (Hopper and Blackwell). Table~\ref{tab:gpu_specs} summarizes their specifications. All GPUs are provisioned via Modal cloud with at least 16 CPU cores.

\paragraph{Software.}
PyTorch 2.11.0~\citep{paszke2019pytorch}, CUDA 13.0, Triton 3.6, and FlashInfer 0.6.9~\citep{ye2025flashinfer}.
To measure kernel runtime, we use NVIDIA CUDA Profiling Tools Interface (CUPTI) in the single-GPU setting, and CUDA events in the multi-GPU setting, for compatibility.
All kernels are warmed up for 25 iterations before timing.
Inputs and weights are in BF16.

\paragraph{Workload configuration.}
The main text focuses on the decode-centric configuration ($D=4{,}096; V=151{,}936$),
which matches models such as Qwen3-8B and Qwen3-235B-A22B MoE~\citep{yang2025qwen3}. We sweep batch sizes $B\in\{1,2,4,8,16,32,64,128,256\}$. Additional results for a larger configuration $D=8{,}192; V=128{,}256$ show the same qualitative trends (Appendix~\ref{app:large_config}).

We have the following baselines to be compared:
\begin{enumerate}
  \item \textbf{Multinomial Sampling.} This baseline computes the logits using a matmul (cuBLAS), saves materialized logits to HBM, and samples with softmax and multinomial. We apply \texttt{torch.compile} to it, which improves speed by 11\% on average over PyTorch eager (range: 5--17\% across GPUs and batch sizes). Unless explicitly stated, all references to Multinomial Sampling refer to the compiled version.
  \item \textbf{FI1 (FlashInfer top-$k$/top-$p$).} \texttt{top\_k\_top\_p\_sampling\_from\_logits}\footnote{\label{fn:flashinfer}\url{https://docs.flashinfer.ai/api/sampling.html}}, a sampling kernel used by vLLM for top-$k$/top-$p$ decode. Logits are also computed using cuBLAS and materialized on HBM. For fair comparison, we deactivate top-k and top-p to prevent unnecessary work (\texttt{top\_k=-1}, \texttt{top\_p=1.0}).
  \item \textbf{FI2 (FlashInfer Gumbel-Max).} \texttt{sampling\_from\_logits}\textsuperscript{\ref{fn:flashinfer}}, FlashInfer's exact Gumbel-Max sampler on materialized logits (skipping softmax normalization). Logits computed using cuBLAS, and materialized on HBM.
\end{enumerate}

\paragraph{Multi-GPU setting.}
We compare FlashSampling against the 3 baselines in the multi-GPU distributed setting (tensor parallel) with 2, 4, and 8 GPUs.
For these experiments, we use a larger hidden size ($V=128{,}256; D=8{,}192$) to mirror the size of real models run in tensor parallel mode, like Llama3 70B~\citep{grattafiori2024llama} and DeepSeek V3.
The baselines are \texttt{torch.compiled} with full graphs, including the all-gather collective operation\footnote{Implemented with the compile-friendly ops in \texttt{torch.distributed.functional\_collectives}.}.

\subsection{Microbenchmarking Results}
\label{sec:fused_benchmarks}

Table~\ref{tab:speedup_kernel} and Figure~\ref{fig:relative_perf} report FlashSampling speedups relative to the three baselines.
A relative speedup of~\emph{e.g.} 2$\times$ means that FlashSampling takes half the runtime as the baseline.
We observe that FlashSampling is faster than all baselines across all batch sizes on the B200 GPU.
The complete data for all four datacenter GPUs is in Appendix~\ref{app:large_config}.

\begin{figure}[htbp]
\centering
\includegraphics[width=0.48\columnwidth]{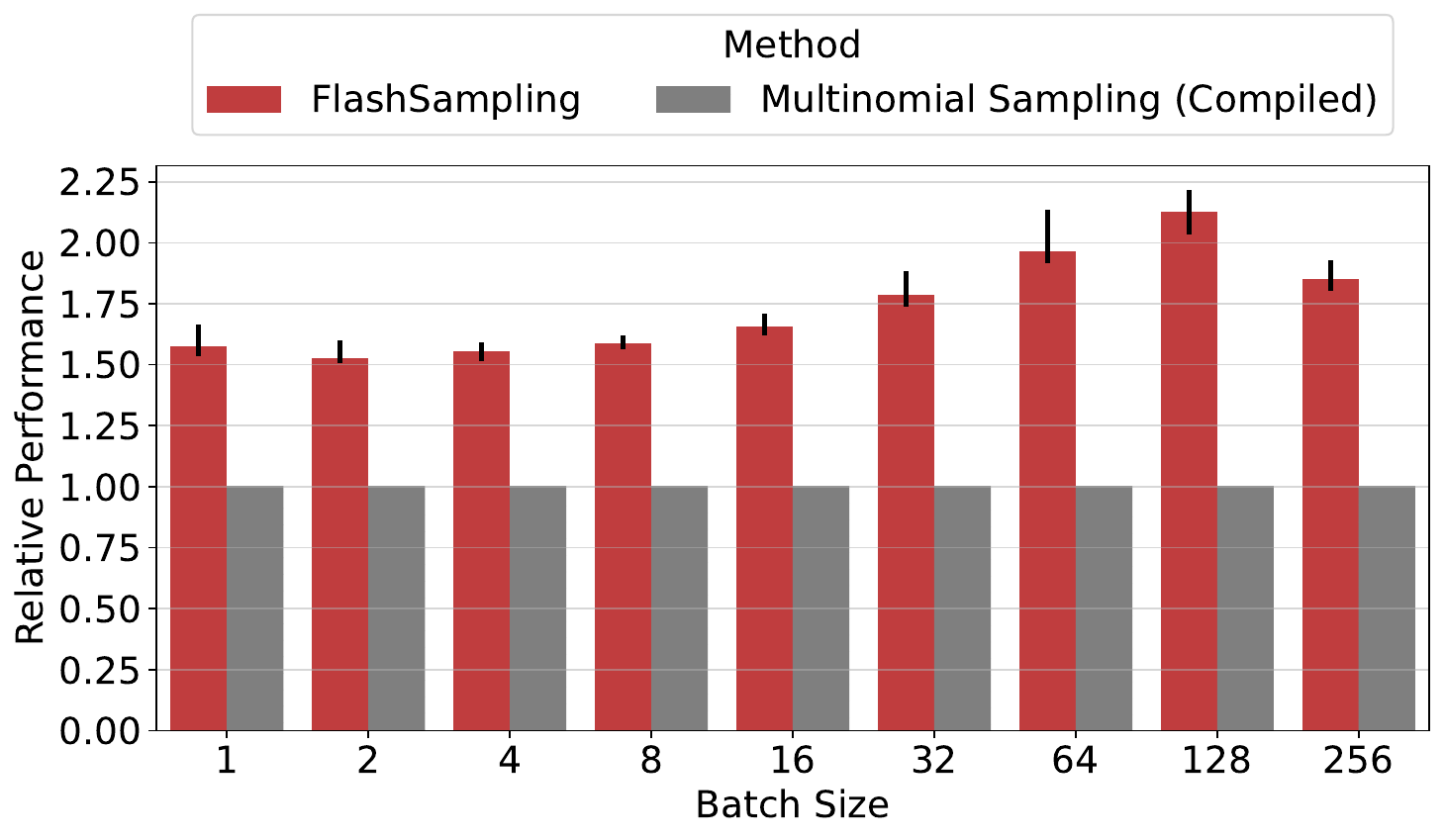}
\hfill
\includegraphics[width=0.48\columnwidth]{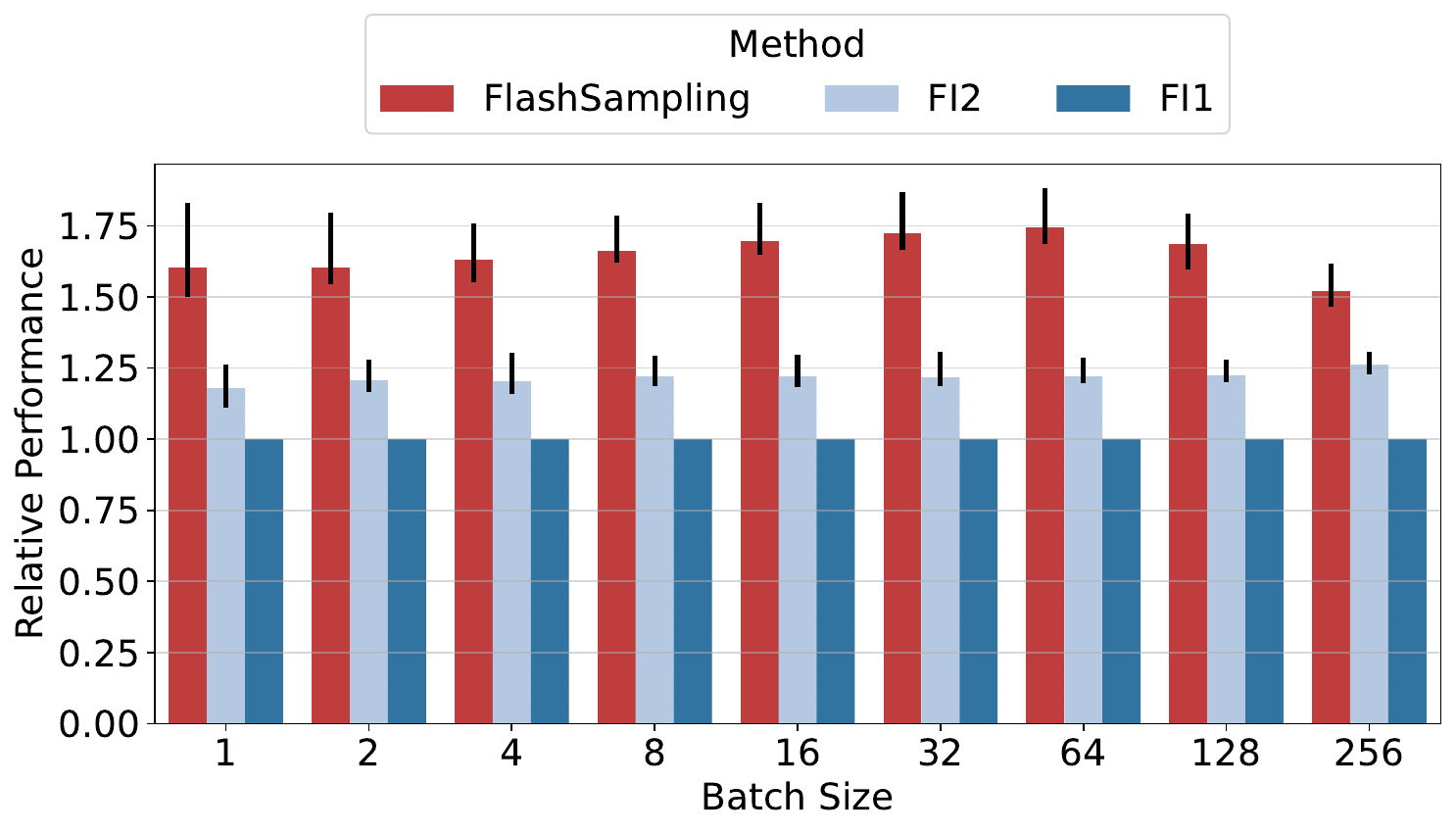}
\caption{Relative performance on B200. Left: Versus Multinomial Sampling. Right: Versus FlashInfer FI1 and FI2. The error bars indicate min-max relative speedups across 10 runs of 100 iterations each.}
\label{fig:relative_perf}
\end{figure}

We have the following observations:
\begin{enumerate}[leftmargin=*, itemsep=1pt, topsep=1pt, parsep=1pt]
  \item \textbf{FlashSampling is consistently faster.} For $B \le 64$, FlashSampling is faster than all baselines on all GPUs. The peak speedup vs.\ Multinomial Sampling is $2.23\times$ (B300) and the peak speedup vs.\ FI1 is $1.74\times$ (B200).
  \item \textbf{The gain is primarily from fusion.} Speedups over FI2 are smaller than speedups over Multinomial Sampling or FI1 because FI2 also uses Gumbel-Max sampling. The remaining gain therefore comes mainly from eliminating logits materialization and sampling overhead (Section~\ref{sec:bsz-trend}).
  \item \textbf{The advantage narrows at large batch sizes.} As batch size grows to $B = 256$, GEMM efficiency matters more and the workload becomes less dominated by memory-bound sampling. Appendix~\ref{app:large_config}, Table~\ref{tab:speedup_large} shows the same qualitative trend for the larger hidden dimension $D=8192$ with the crossover occurring earlier.
\end{enumerate}

\begin{figure}[!htbp]
\centering
\includegraphics[width=1.0\columnwidth]{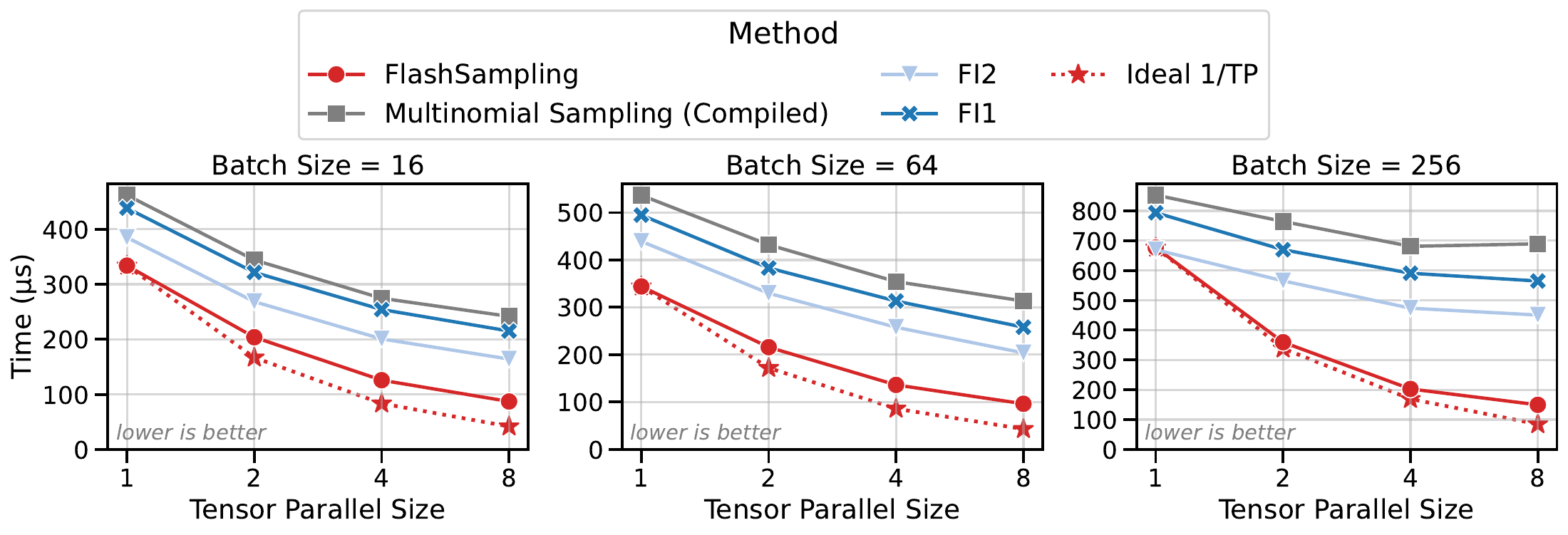}%
\caption{Runtime (lower is better) of FlashSampling and the three baselines at TP=1, 2, 4, 8.
The batch sizes (16, 64, 256) are selected to include compute- and memory-bound decode regimes.
The multi-GPU microbenchmarks took under 5 hours total to run.
The data is tabulated in Table~\ref{tab:tp_runtimes} (Appendix~\ref{app:tp_runtimes}).}
\label{fig:tp_scaling}
\end{figure}

\subsection{Multi-GPU Results with Tensor Parallelism}
\label{sec:tp_results}

\paragraph{Tensor-parallel fusion.}
When the vocabulary is sharded across ranks, each rank can run the fused kernel on its local shard and return only small summaries rather than all local logits. \rev{In our implementation, these are each tile's perturbed maximum and its index (Algorithm~\ref{alg:fused}); the log-mass form of Appendix~\ref{sec:distributed} is not benchmarked.} No $O(V)$ all-gather of logits is required. In practice, this is implemented as per-tile P2P writes from inside the matmul epilogue (Section~\ref{sec:fused_mm_sample}), with a cross-rank barrier before the reduction stage.

We measure the lowest runtime across 5 runs of 100 iterations (the minimum is more robust to one-sided benchmarking noise~\citep{chen2016robust}, which we observed in multi-GPU benchmarks)\footnote{We also observed a large runtime difference between different hardware nodes on Modal cloud, even when GPUs were connected via NVLink.}.

Figure~\ref{fig:tp_scaling} compares the runtime (lower is better) of FlashSampling and the three baselines over tensor parallel sizes $TP \in \{1, 2, 4, 8\}$.
The dotted line shows the \textit{ideal speedup}, which is the runtime for TP=1 divided by the TP size.
We observe that FlashSampling is faster than all baselines in low and medium batch sizes (16 and 64, memory-bound regime).
At higher batch sizes (256, close to the compute-bound regime), FlashSampling closely follows the ideal speedup, scaling optimally across multiple GPUs.
The baselines (FI1, FI2, Multinomial) wait for cuBLAS to finish, then issue an all-gather collective. FlashSampling instead emits per-tile P2P writes, which run concurrently with the matmul, effectively hiding the GPU-to-GPU communication.

\begin{table}[!htbp]
\centering
\caption{Sampling as a percentage of total kernel time.
A high fraction spent on sampling rather than matmul is an indicator of inefficient sampling implementation.
FlashSampling's sampling fraction stays low because it is fused into the matmul epilogue; the baselines' fraction grows with batch size $B$.
Bold marks the highest sampling fraction for each method.}
\label{tab:sampling_pct}
\begin{tabular}{l rr rr rr}
\toprule
& \multicolumn{2}{c}{\emph{FlashSampling}} & \multicolumn{2}{c}{\emph{Multinomial Sampling}} & \multicolumn{2}{c}{\emph{FI2} (Gumbel-Max)} \\
\cmidrule(lr){2-3}\cmidrule(lr){4-5}\cmidrule(lr){6-7}
$B$ & matmul (\%) & sampl. (\%) & matmul (\%) & sampl. (\%) & matmul (\%) & sampl. (\%) \\
\midrule
1   & 97.7 &  2.3 & 93.7 &  6.3 & 94.7 &  5.3 \\
16  & 97.7 &  2.3 & 87.1 & 12.9 & 93.4 &  6.6 \\
64  & 93.6 &  6.0 & 71.3 & \textbf{28.7} & 88.6 & 11.4 \\
256 & 93.4 &  \textbf{6.2} & 73.1 & 26.9 & 88.2 & \textbf{11.8} \\
\bottomrule
\end{tabular}
\end{table}

\subsection{Interpreting the Batch-Size Trend}
\label{sec:bsz-trend}

The cost model in Section~\ref{sec:cost_model} showed that HBM savings from avoiding the logits write and reread alone are small (${\le}6\%$ of traffic).
Figure~\ref{fig:sampling_latency} reveals a larger effect: the baselines' separate sampling kernels are expensive, and their runtime grows steeply with batch size, while FlashSampling absorbs sampling into the matmul at negligible cost (Table~\ref{tab:sampling_pct}: 2--6\% of kernel time).
Eliminating these separate kernels is the primary source of speedup.
The advantage narrows at large batch sizes because FlashSampling's Triton matmul becomes less efficient than cuBLAS (right panel), partially offsetting the sampling savings. Note that Triton is platform-agnostic (AMD, Intel GPUs, etc.), so the cuBLAS gap is a trade-off for portability. Profiling was performed on an RTX~3090 using Nsight Compute and Proton.

\begin{figure}[htbp]
\centering
\includegraphics[width=0.48\columnwidth]{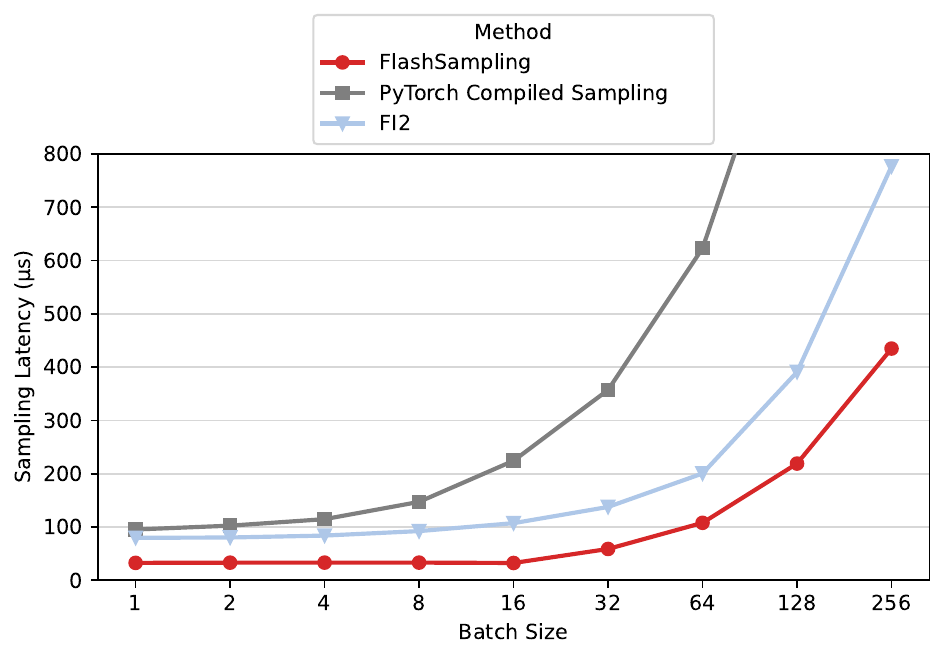}
\hfill
\includegraphics[width=0.48\columnwidth]{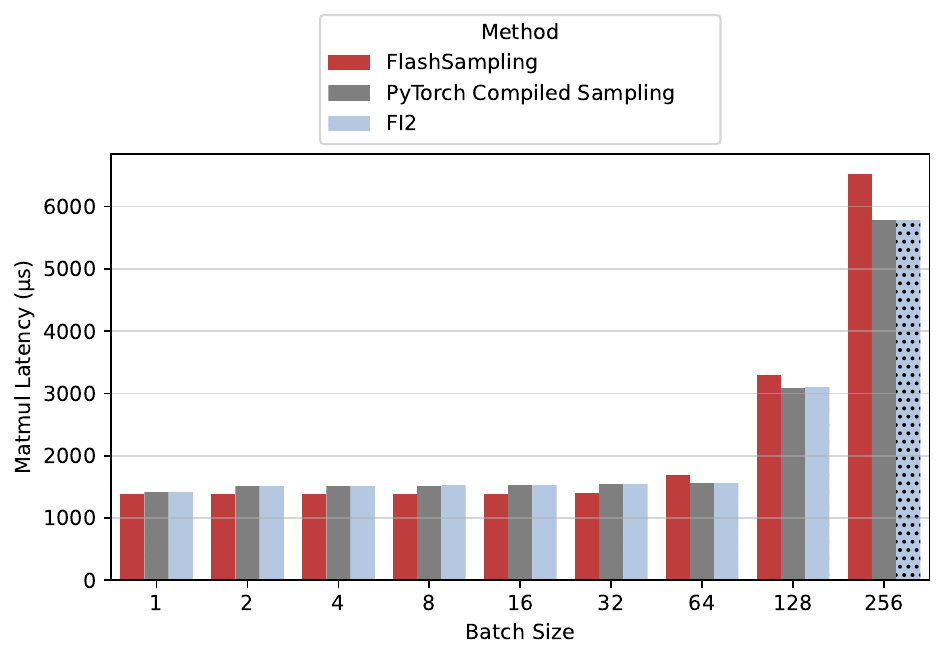}
\caption{Sampling runtime (left) and matmul runtime (right) in $\mu$s vs.\ batch size.
Lower is better.}
\label{fig:sampling_latency}
\end{figure}

\subsection{End-to-End vLLM Evaluation}
\label{sec:vllm}
In this section, we demonstrate the end-to-end speedups achieved by FlashSampling on LLM inference.
We integrate FlashSampling into vLLM~\citep{kwon2023efficient} by replacing the LM-head projection and the sampling step. 
We benchmark TPOT (Time Per Output Token) using problems from the AIME22-24 dataset\footnote{\url{https://huggingface.co/datasets/AI-MO/aimo-validation-aime}}.
vLLM uses continuous batching, so the effective batch size varies dynamically during serving.
We use \texttt{vllm bench sweep serve} with \texttt{--max-concurrency}=$B$ to implement the batch size, and set \texttt{--request-rate}=$B$ for requests to follow a Poisson process at $B$ requests per second.
\rev{All runs use temperature $0.6$, no truncation, and 256 output tokens.}
We rerun the benchmark 5 times for each batch size, compare TPOT between baseline and FlashSampling, and report the median TPOT reduction across the 5 runs.
Experiments run on B200 GPUs across four models spanning a range of sizes: Qwen3-1.7B and Qwen3-8B on a single GPU (TP1), and Qwen3-32B and Llama-3.3-70B on two GPUs (TP2). Tables~\ref{tab:vllm_tpot_abs} and~\ref{tab:vllm_tpot} and Figure~\ref{fig:tpot} present the results.

\paragraph{Key observations.}
The speedups are proportional to the decoding time spent on the LM head compared to attention and FFN.
This explains the highest speedups on Qwen3-1.7B and 8B, which see up to $10.2\%$ and $8.7\%$ TPOT reduction, respectively.
For Qwen3-32B and Llama-3.3-70B, attention and FFN layers dominate decode time, so the speedups are smaller, but consistent (peaks of $2.9\%$ and $2.7\%$, respectively).

\begin{figure}[htbp]
\centering
\includegraphics[width=0.25\columnwidth]{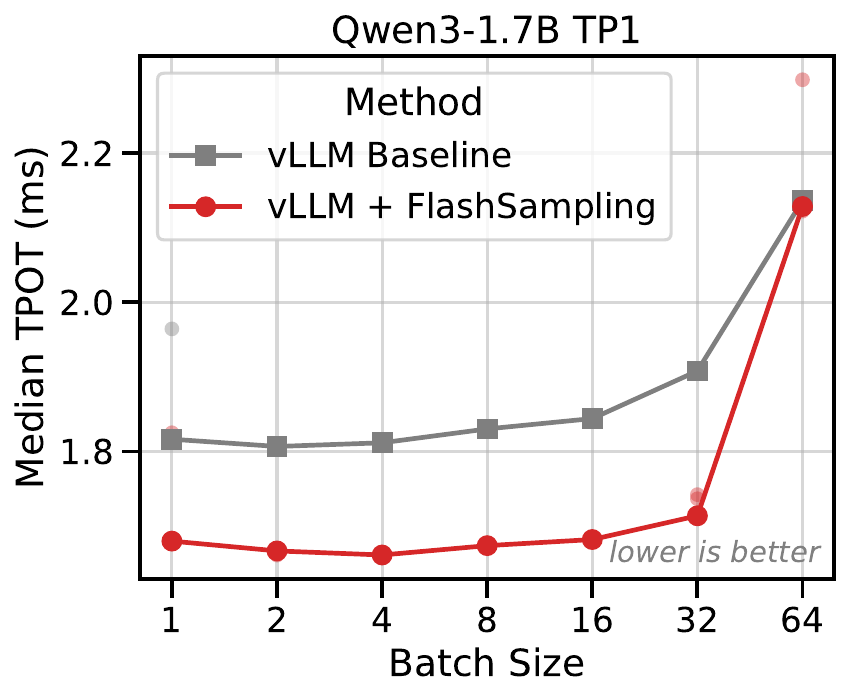}%
\includegraphics[width=0.25\columnwidth]{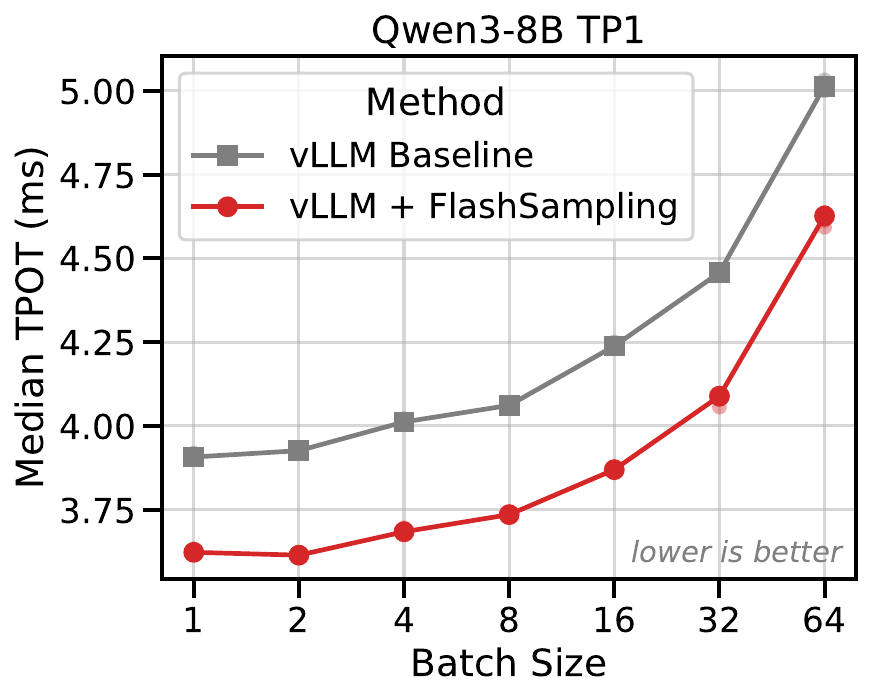}%
\includegraphics[width=0.25\columnwidth]{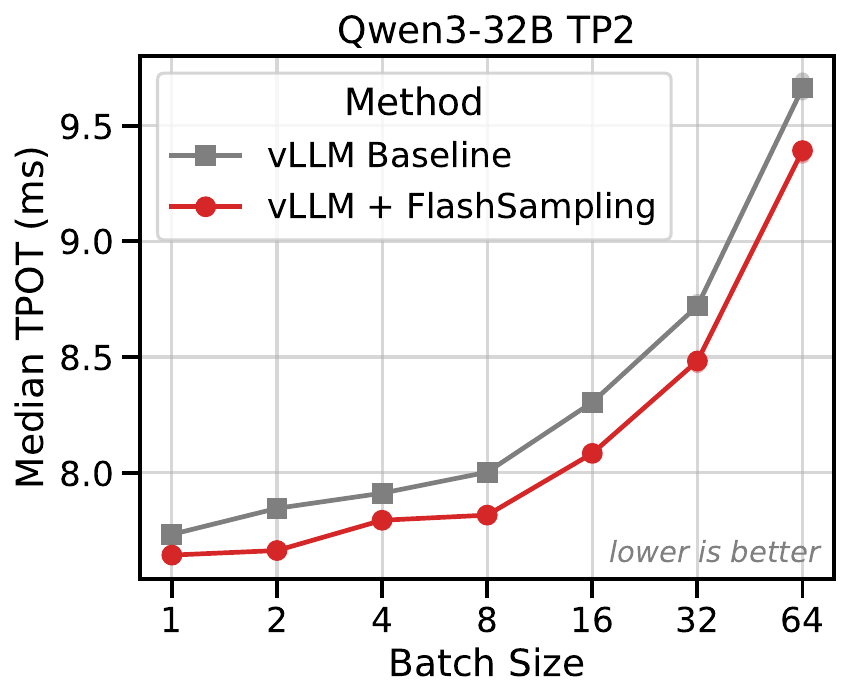}%
\includegraphics[width=0.25\columnwidth]{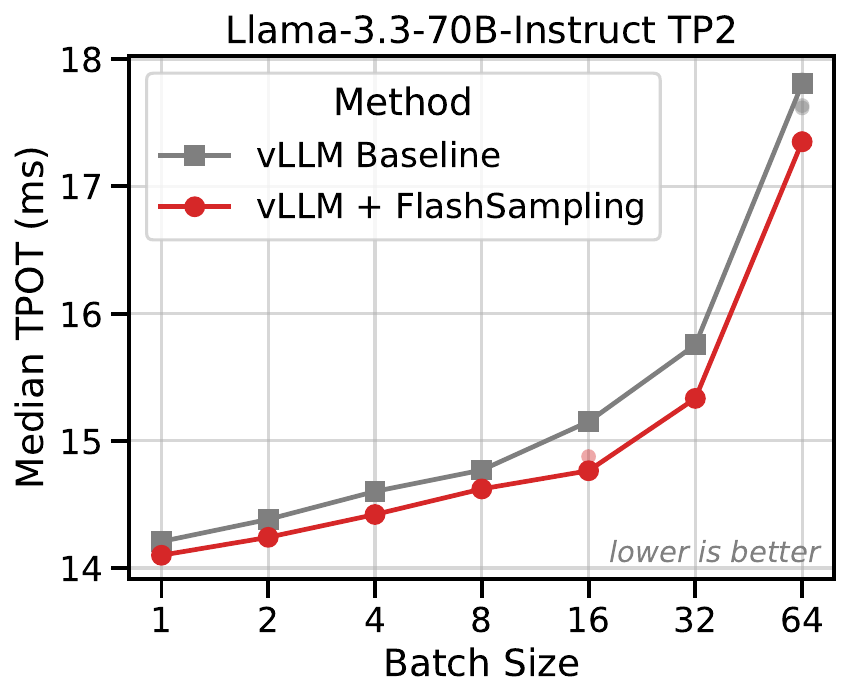}%
\caption{TPOT vs.\ concurrency on B200 across four model sizes. Qwen3-1.7B (TP1) and Qwen3-8B (TP1) see the largest reductions (up to $10\%$ and $7$--$9\%$, respectively). Qwen3-32B (TP2) and Llama-3.3-70B (TP2) achieve smaller but consistent gains (peaks of $2.9\%$ and $2.7\%$), since attention and FFN dominate decode time at these scales.}
\label{fig:tpot}
\end{figure}

\subsection{Empirical Correctness Verification}
\label{sec:correctness}
\paragraph{Kernel Level.}
To verify sampling correctness, we compare samples from FlashSampling to the reference PyTorch implementation using a chi-squared goodness-of-fit test.
With $V=512$ and 10{,}000 samples\rev{, and with $V=128{,}000$ and 10 million samples ($p=0.65$),} we find no statistically significant difference.

\paragraph{End-to-end Level.}
We run FlashSampling on 1,319 questions from the GSM8K dataset~\citep{cobbe2021training} using Qwen3-1.7B and check the answers with an LLM judge.
FlashSampling \rev{(temperature $0.6$, no truncation)} achieves $89.4\%$ accuracy versus $89.6\%$ for the baseline \rev{with Qwen3's default truncation (top-$k$=20, top-$p$=0.95) and $88.4\%$ for the baseline without truncation}.
\rev{Neither difference is statistically significant (paired bootstrap, $p=0.776$ and $p=0.181$), which} is consistent with exact sampling.

\section{Related Work}

\paragraph{Gumbel-Max and Extensions.}
The Gumbel-Max trick for exact categorical sampling dates to \citet{gumbel1954statistical} and was formalized by \citet{maddison2014astar}.
\citet{jang2017gumbelsoftmax} introduced the Gumbel-Softmax relaxation for differentiable discrete sampling, which complements our focus on exact sampling.
\citet{huijben2022review} surveys the broader Gumbel-Max literature.
\citet{kool2019stochastic} extend the trick to top-$k$ sampling without replacement, and \citet{qi2020fast} study fast Gumbel variate generation.
\citet{ahmed2026entropyaligneddecodinglmsbetter} modify the sampling distribution via entropy-aware reweighting and use Gumbel-Max as a subroutine.
FlashSampling contributes a systems-oriented hierarchical decomposition for exact online and distributed sampling in LLM inference, preserving the original distribution exactly.

\paragraph{IO-Aware Kernel Fusion.}
FlashAttention \citep{dao2022flashattention} showed that avoiding materialization of the attention matrix can substantially reduce HBM traffic, with subsequent work improving parallelism \citep{dao2024flashattention} and exploiting hardware asynchrony \citep{shah2024flashattention}.
Cut Your Losses \citep{wijmans2025cut}, Liger Kernel~\citep{hsu2025ligerkernel}, and \citet{dong2025projectionpredictionlogitsscalable} apply the same idea to training-time cross-entropy by fusing the LM-head matmul with the loss computation.
The same matmul-plus-epilogue fusion pattern appears in MLP layers \citep{zhang2026deepkernelfusiontransformers}, RNNs \citep{poppel2025flashrnn}, and whole-model inference \citep{nrusimha2025flashformerwholemodelkernelsefficient}.
At the compiler level, EVT \citep{chen2024evt} auto-generates fused GEMM epilogues via CUTLASS, and \citet{samaga2025fasterapproxtopkharnessing} fuses approximate top-$k$ selection into the matmul on TPUs.
FlashSampling applies this methodology to a different domain: inference-time sampling, exploiting domain-specific structure (Gumbel-Max decomposability), and achieving exactness (no approximations).

\paragraph{Efficient LLM Sampling.}
FlashInfer \citep{ye2025flashinfer} provides optimized GPU kernels for attention and sampling in LLM serving, including sorting-free rejection sampling for top-$k$/top-$p$.
Qrita \citep{park2026qritahighperformancetopktopp} speeds up top-$k$/top-$p$ truncation, which is orthogonal to FlashSampling and theoretically composable with it.
Min-$p$ sampling \citep{minh2025turning} proposes a dynamic truncation method that, like top-$p$, requires probability computation before truncation.
SIMPLE \citep{zhao2025simpledisaggregatingsamplinggpu} offloads sampling to the CPU, motivated by the same bottleneck FlashSampling addresses\rev{, whereas FlashSampling keeps it on the GPU}.
Sampled softmax \citep{rawat2019sampled} reduces large-vocabulary cost by computing the loss over a random subset, trading exactness for speed.
All these methods operate on pre-materialized logits, while FlashSampling avoids materializing them entirely and introduces no approximation.

\section{Conclusion}

We presented \textbf{FlashSampling}, a simple fused design for exact categorical sampling that avoids materializing the $[B,V]$ logits tensor in HBM. The key ideas are straightforward: exact sampling does not require an explicit softmax, the fused tiled kernel is exact by $\argmax$ decomposition over vocabulary tiles, and grouped log-masses yield exact online and distributed variants. The method introduces no approximation: it produces exact samples from the target categorical distribution.
Empirically, FlashSampling is most effective in the memory-bound decode regime, where it absorbs the sampling kernel into the matmul. 
In multi-GPU settings, it overlaps cross-GPU communication with logit computation and HBM loads, scaling near-ideally to 8 GPUs.

\section*{Acknowledgements}
We sincerely thank Yongye Zhu, Zhuoqing Song, and Mayank Mishra for their helpful discussions and constructive feedback. We used large language models to assist in polishing the writing of this work.

\vspace{5ex}
\bibliographystyle{plainnat}
\bibliography{reference}

\clearpage
\appendix

\renewcommand{\appendixpagename}{\centering \huge Appendix}
\appendixpage
\counterwithin{theorem}{section}
\counterwithin{algorithm}{section}

\startcontents[section]
\printcontents[section]{l}{1}{\setcounter{tocdepth}{2}}
\clearpage

\section{Baseline Algorithm}

\begin{algorithm}[ht]
\caption{One common materialized-logits sampling pipeline}
\label{alg:standard_sampling}
\begin{algorithmic}[1]
\Require Hidden state $\vct{h}\in\R^D$, LM-head weights $\mat{W}\in\R^{V\times D}$, optional deterministic transforms
\Ensure Sampled index $i^\star\in\{1,\dots,V\}$
\State $\vct{\ell} \gets \mat{W}\vct{h}$ \Comment{GEMM: compute logits and write to HBM}
\State $\tilde{\vct{\ell}} \gets \mathrm{transform}(\vct{\ell})$ \Comment{temperature, bias, mask; read/write HBM}
\State $m \gets \max_i \tilde{\ell}_i$ \Comment{pass 1 over transformed logits}
\State $Z \gets \sum_{i=1}^V \exp(\tilde{\ell}_i - m)$ \Comment{pass 2 over transformed logits}
\State $p_i \gets \exp(\tilde{\ell}_i - m)/Z$ for all $i$ \Comment{write probabilities}
\State $c_i \gets \sum_{j=1}^i p_j$ for all $i$ \Comment{prefix sum}
\State Draw $u\sim \Unif(0,1)$
\State $i^\star \gets \min\{i : c_i \ge u\}$ \Comment{search}
\State \Return $i^\star$
\end{algorithmic}
\end{algorithm}

\section{GPU Configuration}
\subsection{GPU Memory Hierarchy}
\label{sec:mem_hierarchy}
Table~\ref{tab:mem_hierarchy} summarizes the GPU memory hierarchy. On-chip memory (registers, SRAM) is orders of magnitude faster than HBM but far smaller. FlashSampling exploits this gap by keeping logits in registers/SRAM and never writing the full logits tensor to HBM.

\begin{table}[ht]
\centering
\caption{GPU memory hierarchy (H100 SXM)~\citep{nvidia_h100_whitepaper,nvidia_h100_datasheet}.}
\label{tab:mem_hierarchy}
\begin{tabular}{lcc}
\toprule
\textbf{Level} & \textbf{Capacity} & \textbf{Bandwidth} \\
\midrule
Registers/SRAM & 256\,KB / SM & $>$\,100\,TB/s \\
L2 cache       & 50\,MB        & ${\sim}$\,12\,TB/s \\
HBM3           & 80\,GB        & 3.35\,TB/s \\
\bottomrule
\end{tabular}
\end{table}

\begin{table}[htbp]
\centering
\caption{GPU specifications. Peak BF16 TFLOP/s are dense (without structured sparsity), since the LM-head matmul is a dense GEMM. The ops:byte ratio (peak compute / bandwidth) contextualizes the crossover between bandwidth- and compute-limited regimes, although the exact crossover is kernel-dependent.}
\label{tab:gpu_specs}
\begin{tabular}{lcccc}
\toprule
& \textbf{H100} & \textbf{H200} & \textbf{B200} & \textbf{B300} \\
\midrule
Architecture & Hopper & Hopper & Blackwell & Blackwell \\
HBM capacity (GB) & 80 & 141 & 192 & 288 \\
HBM bandwidth (TB/s) & 3.35 & 4.8 & 8.0 & 8.0 \\
Peak BF16 dense (TFLOP/s) & 989 & 989 & 2,250 & 2,250 \\
Ops:byte ratio & 295 & 206 & 281 & 281 \\
\bottomrule
\end{tabular}
\end{table}

\section{Further details of FlashSampling}
\label{sec:details of flashsampling}

\paragraph{RNG determinism.}
For reproducibility, RNG streams are indexed by the logical output position $(b,i)$ using a counter-based RNG (e.g.\ Philox), so each random number is a deterministic function of a key and a counter.
Uniform variates are mapped to the open interval $(0,1)$ to avoid infinities in the Gumbel transform $g=-\log(-\log u)$.

\paragraph{Numerical precision.}
GEMM accumulation and perturbed scores are computed in FP32 for stability, even when inputs are FP16 or BF16.
Gumbel noise is likewise generated in FP32 to avoid numerical error in the logarithms.
The overhead is minor compared with the GEMM itself.

\section{Theoretical Analysis of FlashSampling}
\label{sec:theory}

This section separates the two exactness arguments used in the paper. The fused tiled kernel is exact \emph{pathwise}: once perturbed scores are formed, the global maximizer is exactly the maximizer of the tile-local maxima. Grouped, online, and distributed variants are exact \emph{in distribution}: they rely on hierarchical factorization through group log-masses.

\subsection{Group-Gumbel-Max: Hierarchical Exact Sampling}
\label{sec:ggm}

Partition $[V]$ into $m$ disjoint groups $\{\mathcal{G}_k\}_{k=0}^{m-1}$; the groups need not have equal size. For any group with at least one finite transformed logit, define
\[
L_k \;=\; \log \sum_{i\in\mathcal{G}_k} \exp(\tilde{\ell}_i)
\;=\;
\lse(\tilde{\vct{\ell}}_{\mathcal{G}_k}).
\]
If a group contains no finite transformed logit, then $L_k=-\infty$, the group has zero probability mass, and it can be skipped.

After discarding zero-mass groups, the categorical distribution factorizes as
\[
\underbrace{\mathbb{P}(K=k)}_{\text{choose group}} \propto \exp(L_k),
\qquad
\underbrace{\mathbb{P}(I=i \mid K=k)}_{\text{choose within group}} \propto \exp(\tilde{\ell}_i)
\quad \text{for } i\in\mathcal{G}_k.
\]
Thus exact sampling from the full categorical can be implemented by first choosing a group using the logits $\{L_k\}$ and then sampling within the chosen group.

\paragraph{Parallel FlashSampling.}
Suppose logits arise from a linear projection $\vct{y}=\mat{W}\vct{x}$, where $\mat{W}\in\R^{V\times D}$ and $\vct{x}\in\R^D$. Let $\mat{W}_{\mathcal{G}_k}\in\R^{|\mathcal{G}_k|\times D}$ be the block of rows indexed by group $\mathcal{G}_k$, so $\vct{y}_k=\mat{W}_{\mathcal{G}_k}\vct{x}\in\R^{|\mathcal{G}_k|}$ are the group logits. Parallel FlashSampling computes groups independently: each group with nonzero mass computes (i) an exact local sample $z_k\sim\cat(\softmax(\vct{y}_k))$ and (ii) its group log-mass $L_k=\lse(\vct{y}_k)$. The algorithm then samples $K\sim\cat(\softmax(\vct{L}))$ and returns $z_K$ mapped to its global index. This is exact by direct factorization.

\paragraph{Online FlashSampling.}
When memory is the primary constraint, FlashSampling can stream groups one at a time and maintain only a running log-mass and a running sample. Suppose the current running state is $(L_{\mathrm{run}},z)$ and the next nonzero-mass group has log-mass $L_k$ and exact local sample $z_k$. Define
\[
L_{\mathrm{new}}=\log\big(e^{L_{\mathrm{run}}}+e^{L_k}\big).
\]
Then replace $z$ by $z_k$ with probability
\[
\frac{e^{L_k}}{e^{L_{\mathrm{run}}}+e^{L_k}}
= e^{L_k-L_{\mathrm{new}}}
= \frac{1}{1+e^{L_{\mathrm{run}}-L_k}},
\]
and otherwise keep $z$. Section~\ref{sec:theory_ggm} proves that this binary merge rule preserves exactness by induction.

\subsection{Distributed FlashSampling for Tensor-Parallel Vocabularies}
\label{sec:distributed}

In tensor-parallel LM heads, the vocabulary dimension is sharded across $n$ GPUs. Naively, each GPU computes local logits and then an all-gather concatenates the full $V$ logits before sampling, incurring communication proportional to the vocabulary size per row. FlashSampling treats shards as groups: each rank returns (i) a local exact sample from its shard, if its shard has nonzero mass for that row, and (ii) the shard log-mass $L_k$. A final exact categorical sample over the shard log-masses chooses which rank provides the global sample. Communication therefore scales with the number of shards, not the number of vocabulary entries.
Mechanically, each rank fan-outs its per-tile candidates to peers via NVLink P2P from within the matmul epilogue, so this communication overlaps with the GEMM rather than executing as a separate post-GEMM collective.

\subsection{A Unifying View: Max-Stability of Grouped Gumbel Perturbations}
\label{sec:max_stability}

Group-Gumbel-Max and FlashSampling both rely on the same structural fact: \emph{max} decomposes over partitions. For grouped variants we additionally use the max-stability of Gumbel perturbations.

\begin{lemma}[Gumbel max-stability under grouping]
\label{lem:max_stability}
Let $\{g_i\}_{i=1}^V$ be i.i.d.\ $\gumbel(0,1)$ and let $\{\mathcal{G}_k\}_{k=0}^{m-1}$ be a partition of $[V]$. Assume each group under discussion contains at least one finite transformed logit. Define
\[
M_k \;=\; \max_{i\in\mathcal{G}_k}(\tilde{\ell}_i + g_i),
\qquad
I_k \;=\; \argmax_{i\in\mathcal{G}_k}(\tilde{\ell}_i + g_i),
\qquad
L_k \;=\; \log\sum_{i\in\mathcal{G}_k} e^{\tilde{\ell}_i}.
\]
Then:
\begin{enumerate}
  \item $M_k \sim \gumbel(L_k,1)$,
  \item $\{M_k\}$ are independent across disjoint groups,
  \item $\mathbb{P}(I_k=i)=e^{\tilde{\ell}_i}/\sum_{j\in\mathcal{G}_k}e^{\tilde{\ell}_j}$ for $i\in\mathcal{G}_k$.
\end{enumerate}
\end{lemma}
\begin{proof}
For any real $t$,
\[
\mathbb{P}(M_k\le t)
=
\prod_{i\in\mathcal{G}_k}\mathbb{P}(g_i\le t-\tilde{\ell}_i)
=
\prod_{i\in\mathcal{G}_k}\exp\big(-e^{-(t-\tilde{\ell}_i)}\big)
=
\exp\Big(-e^{-(t-L_k)}\Big),
\]
which is the CDF of $\gumbel(L_k,1)$. Independence follows because the groups are disjoint and the underlying Gumbels are independent. The within-group argmax probabilities are exactly the Gumbel-Max trick applied to the restricted transformed logits.
\end{proof}

\paragraph{Consequence.}
For grouped variants, selecting a group by $\argmax_k M_k$ is equivalent in distribution to applying Gumbel-Max directly to the group logits $\{L_k\}$. The outer group sample may therefore use fresh independent Gumbels, or it may reuse explicitly computed group maxima. For the fused two-stage kernel in Algorithm~\ref{alg:fused}, exactness does \emph{not} rely on max-stability: once the perturbed scores $x_i=\tilde{\ell}_i+g_i$ have been formed, exactness is simply the deterministic identity
\[
\max_i x_i = \max_t \max_{i\in\mathcal{V}_t} x_i.
\]

\subsection{Exactness of Group-Gumbel-Max}
\label{sec:theory_ggm}

The correctness of grouped FlashSampling rests on two facts: exact group factorization, and the binary merge rule used by the online variant.

\begin{lemma}[Exact group factorization]
\label{lem:group_factorization}
Let $[V]$ be partitioned into groups $\{\mathcal{G}_k\}_{k=0}^{m-1}$, and discard any zero-mass groups. Define $L_k=\log\sum_{i\in\mathcal{G}_k}\exp(\tilde{\ell}_i)$. If we sample $K\sim \cat(\softmax(\vct{L}))$ and then sample $I\mid(K=k)\sim \cat(\softmax(\tilde{\vct{\ell}}_{\mathcal{G}_k}))$, the marginal distribution of $I$ equals $\cat(\softmax(\tilde{\vct{\ell}}))$.
\end{lemma}
\begin{proof}
For any $i\in\mathcal{G}_k$,
\[
\mathbb{P}(I=i)
=
\mathbb{P}(K=k)\,\mathbb{P}(I=i\mid K=k)
=
\frac{e^{L_k}}{\sum_{s}e^{L_s}}
\cdot
\frac{e^{\tilde{\ell}_i}}{\sum_{j\in\mathcal{G}_k}e^{\tilde{\ell}_j}}
=
\frac{e^{\tilde{\ell}_i}}{\sum_{j=1}^V e^{\tilde{\ell}_j}}.
\]
\end{proof}

\begin{lemma}[Binary merge rule]
\label{lem:binary_merge}
Let $A,B\subseteq [V]$ be disjoint and suppose both have nonzero mass. Define
\[
L_A=\log\sum_{i\in A} e^{\tilde{\ell}_i},
\qquad
L_B=\log\sum_{i\in B} e^{\tilde{\ell}_i}.
\]
Suppose $Z_A\sim\cat(\softmax(\tilde{\vct{\ell}}_A))$, $Z_B\sim\cat(\softmax(\tilde{\vct{\ell}}_B))$, and an independent Bernoulli choice selects $B$ with probability $e^{L_B}/(e^{L_A}+e^{L_B})$. Returning $Z_B$ when $B$ is selected and $Z_A$ otherwise yields an exact sample from $\cat(\softmax(\tilde{\vct{\ell}}_{A\cup B}))$.
\end{lemma}
\begin{proof}
For any $i\in A$,
\[
\mathbb{P}(Z=i)
=
\mathbb{P}(\text{choose }A)\,\mathbb{P}(Z_A=i)
=
\frac{e^{L_A}}{e^{L_A}+e^{L_B}}
\cdot
\frac{e^{\tilde{\ell}_i}}{\sum_{j\in A} e^{\tilde{\ell}_j}}
=
\frac{e^{\tilde{\ell}_i}}{\sum_{j\in A\cup B} e^{\tilde{\ell}_j}}.
\]
The same calculation for $i\in B$ gives
\[
\mathbb{P}(Z=i)
=
\frac{e^{L_B}}{e^{L_A}+e^{L_B}}
\cdot
\frac{e^{\tilde{\ell}_i}}{\sum_{j\in B} e^{\tilde{\ell}_j}}
=
\frac{e^{\tilde{\ell}_i}}{\sum_{j\in A\cup B} e^{\tilde{\ell}_j}}.
\]
Hence $Z\sim\cat(\softmax(\tilde{\vct{\ell}}_{A\cup B}))$.
\end{proof}

\begin{theorem}[Exactness of hierarchical FlashSampling]
\label{thm:ggm_exact}
Algorithms~\ref{alg:flash_parallel}, \ref{alg:flash_seq}, and \ref{alg:flash_dist} return an exact sample from $\cat(\softmax(\tilde{\vct{\ell}}))$.
\end{theorem}
\begin{proof}
For the parallel and distributed variants, Lemma~\ref{lem:group_factorization} shows that it suffices to sample the group or shard index from logits $\{L_k\}$ and then sample within the chosen group; both steps are exact.

For the online variant, initialize with an exact sample from the first nonzero-mass group. Each subsequent update merges the current union with the next nonzero-mass group using Lemma~\ref{lem:binary_merge}. An induction over the streamed groups therefore yields an exact sample from the full categorical distribution.
\end{proof}

\subsection{Exactness of Tile-Wise FlashSampling Reduction}

FlashSampling also relies on a simpler structural lemma: the global maximum equals the maximum of the tile-local maxima.

\begin{lemma}[Max over vocabulary tiles]
\label{lem:max_tiles}
Let $\{x_i\}_{i=1}^V$ be real numbers and let $\{\mathcal{V}_s\}_{s=0}^{n_{\mathrm{tile}}-1}$ be a partition of $[V]$ into vocabulary tiles. For each tile, define
\[
m_s=\max_{i\in\mathcal{V}_s}x_i,
\qquad
\hat{\imath}_s \in \argmax_{i\in\mathcal{V}_s}x_i,
\]
where $\hat{\imath}_s$ is a global index in $\mathcal{V}_s$. Then
\[
\max_{i\in[V]} x_i = \max_s m_s.
\]
Moreover, for any $s^\star\in\argmax_s m_s$, the chosen index $\hat{\imath}_{s^\star}$ is a global maximizer. Conversely, every global maximizer lies in some tile $s^\star\in\argmax_s m_s$.
\end{lemma}
\begin{proof}
The identity for the maximum value is immediate:
\[
\max_{i\in[V]} x_i
=
\max_s \max_{i\in\mathcal{V}_s} x_i
=
\max_s m_s.
\]
If $s^\star\in\argmax_s m_s$, then $x_{\hat{\imath}_{s^\star}}=m_{s^\star}=\max_i x_i$, so $\hat{\imath}_{s^\star}$ is a global maximizer. Conversely, if $i^\star$ is any global maximizer, then its tile $s^\star$ satisfies $m_{s^\star}=x_{i^\star}=\max_i x_i$, hence $s^\star\in\argmax_s m_s$.
\end{proof}

Applying Lemma~\ref{lem:max_tiles} to $x_i=\tilde{\ell}_i+g_i$ justifies the two-stage fused design in Algorithm~\ref{alg:fused}. Because the Gumbel variables are continuous, the global maximizer is unique almost surely, so the tile-wise reduction returns exactly the same index as a full row-wise $\argmax$ with probability one.

\subsection{Top-\texorpdfstring{$k$}{k}, Nucleus Sampling, and Masking}
\label{sec:truncation}
Practical decoding often uses truncated supports, and the tiled structure of FlashSampling naturally accommodates most of them.

\begin{itemize}
  \item \textbf{Top-$k$:} \rev{Top-$k$ truncation keeps the $k$ largest unperturbed logits and then draws one sample from them.
  (The Gumbel-Top-$k$ trick \citep{kool2019stochastic} does something different: it draws $k$ items without replacement.)
  The $k$ largest logits can be found from a partition of the vocabulary: if $k$ groups each contain a logit of at least $\tau$, then at least $k$ logits are at least $\tau$, so every one of the $k$ largest logits is at least $\tau$ and lies in a group whose maximum is at least $\tau$.
  The epilogue therefore only has to write one maximum per group, and a second stage recomputes the logits of the few groups that reach the threshold.}
  \item \textbf{Top-$p$ (nucleus):} Unlike top-$k$, nucleus sampling~\citep{holtzman2019curious} requires a global softmax followed by a sorted cumulative sum, neither of which decomposes into independent tile-local work. However, top-$p$ can be applied \emph{after} top-$k$ on the reduced candidate set of only $k$ elements, where softmax, sorting, and cumulative summation are negligible. This sequential top-$k$-then-top-$p$ strategy is used in practice by vLLM\footnote{\url{https://github.com/vllm-project/vllm/blob/v0.16.0/vllm/v1/sample/ops/topk_topp_sampler.py\#L264-L279}}\textsuperscript{,}\footnote{\url{https://github.com/vllm-project/vllm/blob/v0.16.1rc0/vllm/v1/sample/ops/topk_topp_triton.py\#L956}}, FlashInfer\footnote{\url{https://github.com/flashinfer-ai/flashinfer/blob/v0.6.3/flashinfer/sampling.py\#L1069-L1072}}, and other SOTA top-$k$ top-$p$ algorithms~\citep{park2026qritahighperformancetopktopp}.
  \item \textbf{Masking:} Forbidden indices (e.g.\ banned tokens, grammar constraints) are supported by setting their logits to $-\infty$ before perturbation, which preserves exactness over the restricted support.
\end{itemize}

\noindent \rev{Appendix~\ref{app:topk} describes and measures an implementation of top-$k$ followed by top-$p$.}

\section{Kernel Microbenchmark Data}
\label{app:large_config}

For completeness, Table~\ref{tab:speedup_kernel} reports the smaller-configuration kernel results deferred from the main text, and Table~\ref{tab:speedup_large} reports the larger configuration. The same qualitative pattern appears in both: FlashSampling is strongest in the small-batch decode regime, while the advantage narrows once the workload becomes more GEMM-efficiency dominated.

\begin{table}[htbp]
\centering
\small
\setlength{\tabcolsep}{3.8pt}
\renewcommand{\arraystretch}{0.97}
\caption{FlashSampling relative speedup vs.\ three baselines on the smaller configuration ($D{=}4096$, $V{=}151\text{k}$). Values $>1$ indicate FlashSampling is faster; bold marks the peak per GPU within each baseline. The numbers are medians over 100 iterations. Each column takes under 10 minutes to run.}
\label{tab:speedup_kernel}
\begin{tabular}{l cccc cccc cccc}
\toprule
& \multicolumn{4}{c}{\emph{vs.\ Multinomial Sampling}} & \multicolumn{4}{c}{\emph{vs.\ FI1} (top-$k$/top-$p$)} & \multicolumn{4}{c}{\emph{vs.\ FI2} (Gumbel-Max)} \\
\cmidrule(lr){2-5}\cmidrule(lr){6-9}\cmidrule(lr){10-13}
$B$ & H100 & H200 & B200 & B300 & H100 & H200 & B200 & B300 & H100 & H200 & B200 & B300 \\
\midrule
1   & 1.29 & 1.35 & 1.58 & 1.55 & 1.34 & 1.39 & 1.65 & 1.57 & 1.20 & 1.24 & 1.35 & 1.33 \\
2   & 1.28 & 1.34 & 1.52 & 1.50 & 1.32 & 1.44 & 1.58 & 1.56 & 1.19 & 1.22 & 1.32 & 1.29 \\
4   & 1.28 & 1.35 & 1.58 & 1.54 & 1.32 & 1.42 & 1.62 & 1.61 & 1.19 & 1.23 & 1.37 & 1.32 \\
8   & 1.31 & 1.38 & 1.58 & 1.55 & 1.35 & 1.48 & 1.63 & 1.60 & 1.19 & 1.24 & 1.37 & 1.33 \\
16  & 1.35 & 1.44 & 1.66 & 1.63 & 1.37 & 1.49 & 1.70 & 1.65 & 1.20 & 1.25 & 1.39 & 1.35 \\
32  & 1.44 & 1.53 & 1.79 & 1.74 & 1.39 & \textbf{1.50} & 1.70 & 1.67 & 1.22 & \textbf{1.27} & 1.42 & 1.37 \\
64  & 1.64 & \textbf{1.67} & 1.96 & 1.92 & \textbf{1.47} & 1.49 & \textbf{1.74} & 1.69 & \textbf{1.27} & 1.25 & \textbf{1.43} & 1.39 \\
128 & \textbf{1.67} & 1.43 & \textbf{2.14} & \textbf{2.23} & 1.36 & 1.15 & 1.68 & \textbf{1.74} & 1.14 & 0.94 & 1.39 & \textbf{1.44} \\
256 & 1.30 & 1.22 & 1.84 & 2.03 & 1.03 & 1.00 & 1.54 & 1.65 & 0.81 & 0.79 & 1.19 & 1.30 \\
\bottomrule
\end{tabular}
\end{table}

\begin{table}[ht]
\centering
\small
\setlength{\tabcolsep}{3.8pt}
\renewcommand{\arraystretch}{0.97}
\caption{FlashSampling speedup vs.\ three baselines on the larger configuration ($D{=}8192$, $V{=}128\text{k}$). Values $>1$ indicate FlashSampling is faster; bold marks the peak per GPU within each baseline. At $B{\ge}128$ the advantage narrows and cuBLAS GEMM efficiency becomes increasingly important.}
\label{tab:speedup_large}
\begin{tabular}{l cccc cccc cccc}
\toprule
& \multicolumn{4}{c}{\emph{vs.\ Multinomial Sampling}} & \multicolumn{4}{c}{\emph{vs.\ FI1} (top-$k$/top-$p$)} & \multicolumn{4}{c}{\emph{vs.\ FI2} (Gumbel-Max)} \\
\cmidrule(lr){2-5}\cmidrule(lr){6-9}\cmidrule(lr){10-13}
$B$ & H100 & H200 & B200 & B300 & H100 & H200 & B200 & B300 & H100 & H200 & B200 & B300 \\
\midrule
1   & 1.23 & 1.26 & 1.45 & 1.41 & 1.20 & 1.25 & 1.36 & 1.30 & 1.14 & 1.13 & 1.21 & 1.19 \\
2   & 1.22 & 1.23 & 1.41 & 1.39 & 1.21 & 1.21 & 1.33 & 1.29 & 1.13 & 1.11 & 1.21 & 1.18 \\
4   & 1.22 & 1.24 & 1.35 & 1.34 & 1.20 & 1.21 & 1.31 & 1.28 & 1.13 & 1.12 & 1.15 & 1.13 \\
8   & 1.23 & 1.25 & 1.37 & 1.36 & 1.21 & 1.23 & 1.30 & 1.29 & 1.13 & 1.12 & 1.14 & 1.13 \\
16  & 1.24 & 1.27 & 1.40 & 1.39 & 1.21 & 1.25 & 1.32 & 1.30 & 1.14 & 1.13 & 1.15 & 1.15 \\
32  & 1.27 & 1.34 & 1.47 & 1.45 & 1.23 & 1.30 & 1.36 & 1.35 & 1.15 & 1.18 & 1.20 & 1.19 \\
64  & \textbf{1.36} & \textbf{1.43} & \textbf{1.61} & 1.58 & \textbf{1.28} & \textbf{1.33} & \textbf{1.46} & \textbf{1.44} & \textbf{1.19} & \textbf{1.20} & \textbf{1.30} & \textbf{1.28} \\
128 & 1.30 & 1.01 & 1.59 & \textbf{1.61} & 1.16 & 0.88 & 1.34 & 1.39 & 1.05 & 0.78 & 1.19 & 1.23 \\
256 & 0.88 & 0.82 & 1.27 & 1.32 & 0.80 & 0.74 & 1.12 & 1.19 & 0.69 & 0.65 & 0.97 & 1.02 \\
\bottomrule
\end{tabular}
\end{table}

\section{Multi-GPU Runtime Values}
\label{app:tp_runtimes}
Table~\ref{tab:tp_runtimes} reports the minimum kernel runtime ($\mu$s) underlying Figure~\ref{fig:tp_scaling}, taken as the lowest of 5 runs of 100 iterations each.
The configuration is the larger one ($D{=}8192$, $V{=}128\text{k}$).
FlashSampling attains the lowest runtime in every cell except $(B{=}256, \text{TP}{=}1)$, where FI2 is marginally faster.

\begin{table}[ht]
\centering
\small
\setlength{\tabcolsep}{6pt}
\caption{Minimum kernel runtime ($\mu$s, lower is better) across tensor-parallel sizes $\text{TP}\in\{1,2,4,8\}$ and batch sizes $B\in\{16,64,256\}$ for input shape ($D{=}8192$, $V{=}128\text{k}$). Bold marks the fastest method per ($B$, TP) cell.}
\label{tab:tp_runtimes}
\begin{tabular}{l l rrrr}
\toprule
$B$ & Method & TP=1 & TP=2 & TP=4 & TP=8 \\
\midrule
\multirow{4}{*}{16}
 & FlashSampling   & \textbf{333.8} & \textbf{203.8} & \textbf{125.8} & \textbf{87.1}  \\
 & FI1             & 438.3 & 321.5 & 254.1 & 215.1 \\
 & FI2             & 385.0 & 268.4 & 200.7 & 164.2 \\
 & Multinomial     & 461.8 & 344.1 & 274.5 & 241.7 \\
\midrule
\multirow{4}{*}{64}
 & FlashSampling   & \textbf{344.1} & \textbf{215.9} & \textbf{136.2} & \textbf{96.4}  \\
 & FI1             & 494.8 & 383.1 & 313.3 & 258.0 \\
 & FI2             & 439.2 & 329.7 & 257.9 & 203.8 \\
 & Multinomial     & 536.6 & 432.1 & 354.3 & 313.3 \\
\midrule
\multirow{4}{*}{256}
 & FlashSampling   & 676.8 & \textbf{359.5} & \textbf{202.7} & \textbf{149.5} \\
 & FI1             & 793.7 & 669.7 & 590.8 & 564.7 \\
 & FI2             & \textbf{669.7} & 565.3 & 473.6 & 450.6 \\
 & Multinomial     & 852.9 & 764.5 & 681.0 & 689.4 \\
\bottomrule
\end{tabular}
\end{table}

\section{Absolute/Relative TPOT Results for vLLM Evaluation}
\label{app:vllm_absolute}
Table~\ref{tab:vllm_tpot_abs} reports the median TPOT (ms) for baseline and FlashSampling on B200, complementing the relative speedups in Table~\ref{tab:vllm_tpot}.
Smaller models (Qwen3-1.7B, Qwen3-8B) run on a single GPU (TP1); larger models (Qwen3-32B, Llama-3.3-70B) run on two GPUs (TP2).

\begin{table}[ht]
\centering
\small
\setlength{\tabcolsep}{4pt}
\caption{Median TPOT (ms) over 5 runs on B200 for baseline and FlashSampling. Lower is better. $B$ is the batch size.}
\label{tab:vllm_tpot_abs}
\begin{tabular}{l cc cc cc cc}
\toprule
& \multicolumn{2}{c}{Qwen3-1.7B (TP1)} & \multicolumn{2}{c}{Qwen3-8B (TP1)} & \multicolumn{2}{c}{Qwen3-32B (TP2)} & \multicolumn{2}{c}{Llama-3.3-70B (TP2)} \\
\cmidrule(lr){2-3}\cmidrule(lr){4-5}\cmidrule(lr){6-7}\cmidrule(lr){8-9}
$B$ & Base & FlashSampling & Base & FlashSampling & Base & FlashSampling & Base & FlashSampling \\
\midrule
1   & 1.82 & 1.68 & 3.91 & 3.62 & 7.73 & 7.64 & 14.21 & 14.10 \\
2   & 1.81 & 1.67 & 3.93 & 3.61 & 7.85 & 7.66 & 14.38 & 14.24 \\
4   & 1.81 & 1.66 & 4.01 & 3.68 & 7.91 & 7.79 & 14.60 & 14.42 \\
8   & 1.83 & 1.67 & 4.06 & 3.74 & 8.00 & 7.82 & 14.77 & 14.62 \\
16  & 1.84 & 1.68 & 4.24 & 3.87 & 8.30 & 8.08 & 15.15 & 14.76 \\
32  & 1.91 & 1.71 & 4.46 & 4.09 & 8.72 & 8.48 & 15.75 & 15.33 \\
64  & 2.14 & 2.13 & 5.01 & 4.63 & 9.66 & 9.39 & 17.81 & 17.35 \\
\bottomrule
\end{tabular}
\end{table}

\begin{table}[htbp]
\centering
\caption{TPOT speedup in \% (larger is better) computed as ($1 - \text{FlashSampling}/\text{baseline}$), and standard deviation across 5 runs.
$B$ is the batch size.
Bold marks the peak per model. Absolute TPOT values are in Appendix~\ref{app:vllm_absolute}.
Each column takes under 2 hours to run.
\rev{$^\dagger$The run-to-run deviation of this cell is $3.6\%$, against at most $0.8\%$ in every other cell, and both systems reach almost the same TPOT (Table~\ref{tab:vllm_tpot_abs}); we consider this cell unreliable.}}
\label{tab:vllm_tpot}
\begin{tabular}{l rrrr}
\toprule
$B$ & Qwen3-1.7B (TP1) & Qwen3-8B (TP1) & Qwen3-32B (TP2) & Llama-3.3-70B (TP2) \\
\midrule
1   & $7.5 \pm 0.2$\,\% & $7.3 \pm 0.2$\,\% & $1.2 \pm 0.0$\,\% & $0.8 \pm 0.0$\,\% \\
2   & $7.7 \pm 0.1$\,\% & $7.9 \pm 0.2$\,\% & $2.3 \pm 0.1$\,\% & $1.0 \pm 0.0$\,\% \\
4   & $8.3 \pm 0.0$\,\% & $8.1 \pm 0.1$\,\% & $1.5 \pm 0.0$\,\% & $1.2 \pm 0.0$\,\% \\
8   & $8.5 \pm 0.0$\,\% & $8.1 \pm 0.1$\,\% & $2.3 \pm 0.0$\,\% & $1.0 \pm 0.0$\,\% \\
16  & $8.8 \pm 0.1$\,\% & $\mathbf{8.7} \pm 0.3$\,\% & $2.7 \pm 0.0$\,\% & $2.5 \pm 0.3$\,\% \\
32  & $\mathbf{10.2} \pm 0.8$\,\% & $8.5 \pm 0.3$\,\% & $\mathbf{2.9} \pm 0.2$\,\% & $\mathbf{2.7} \pm 0.1$\,\% \\
64  & $0.4 \pm 3.6$\,\%\rev{$^\dagger$} & $8.0 \pm 0.5$\,\% & $2.8 \pm 0.3$\,\% & $2.5 \pm 0.5$\,\% \\
\bottomrule
\end{tabular}
\end{table}

\FloatBarrier
\section{Roofline Analysis and Bandwidth Utilization}
\label{sec:roofline}

The LM-head projection is memory-bandwidth-bound at small batch sizes because arithmetic intensity equals $B$ (the weight matrix dominates traffic).
Figure~\ref{fig:roofline} confirms this on B200.

\begin{figure}[htbp]
\centering
\includegraphics[width=0.47\columnwidth]{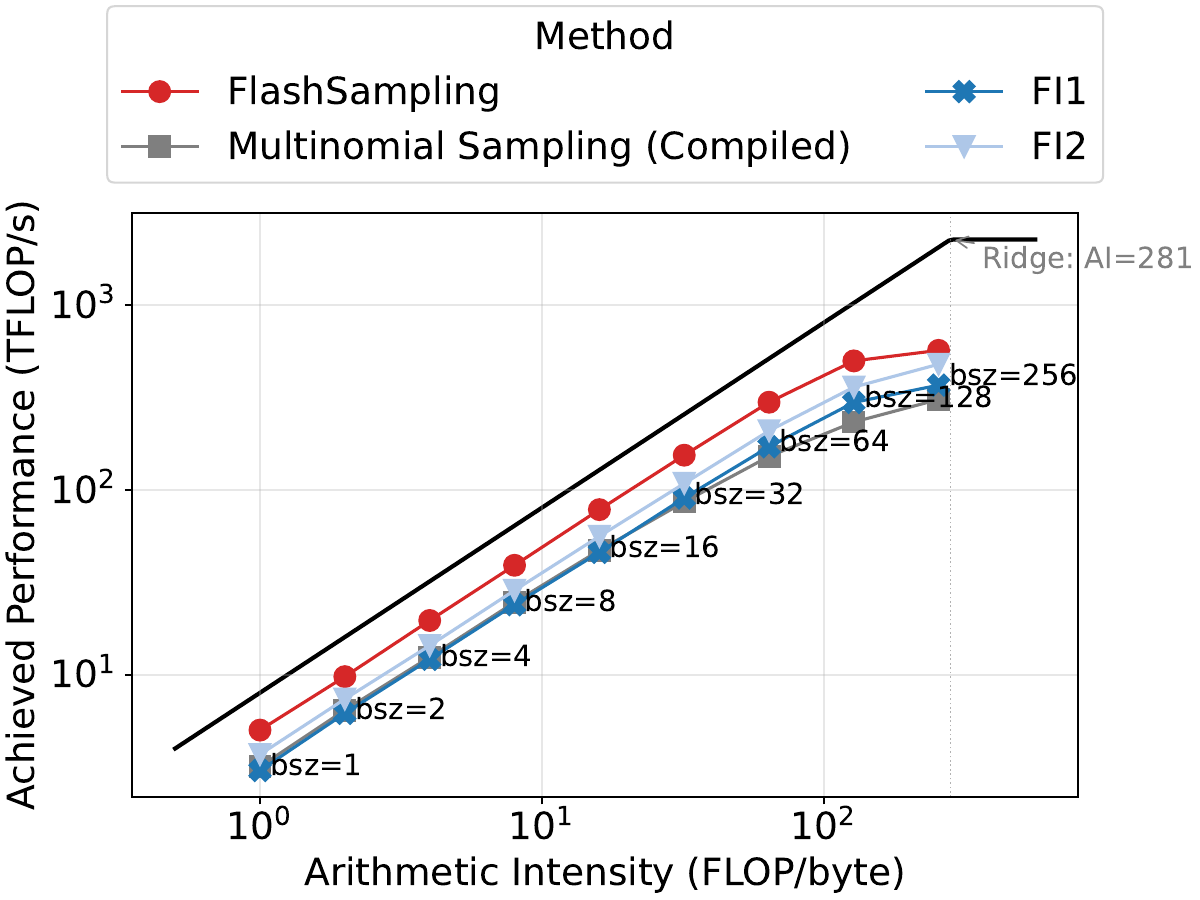}%
\includegraphics[width=0.53\columnwidth]{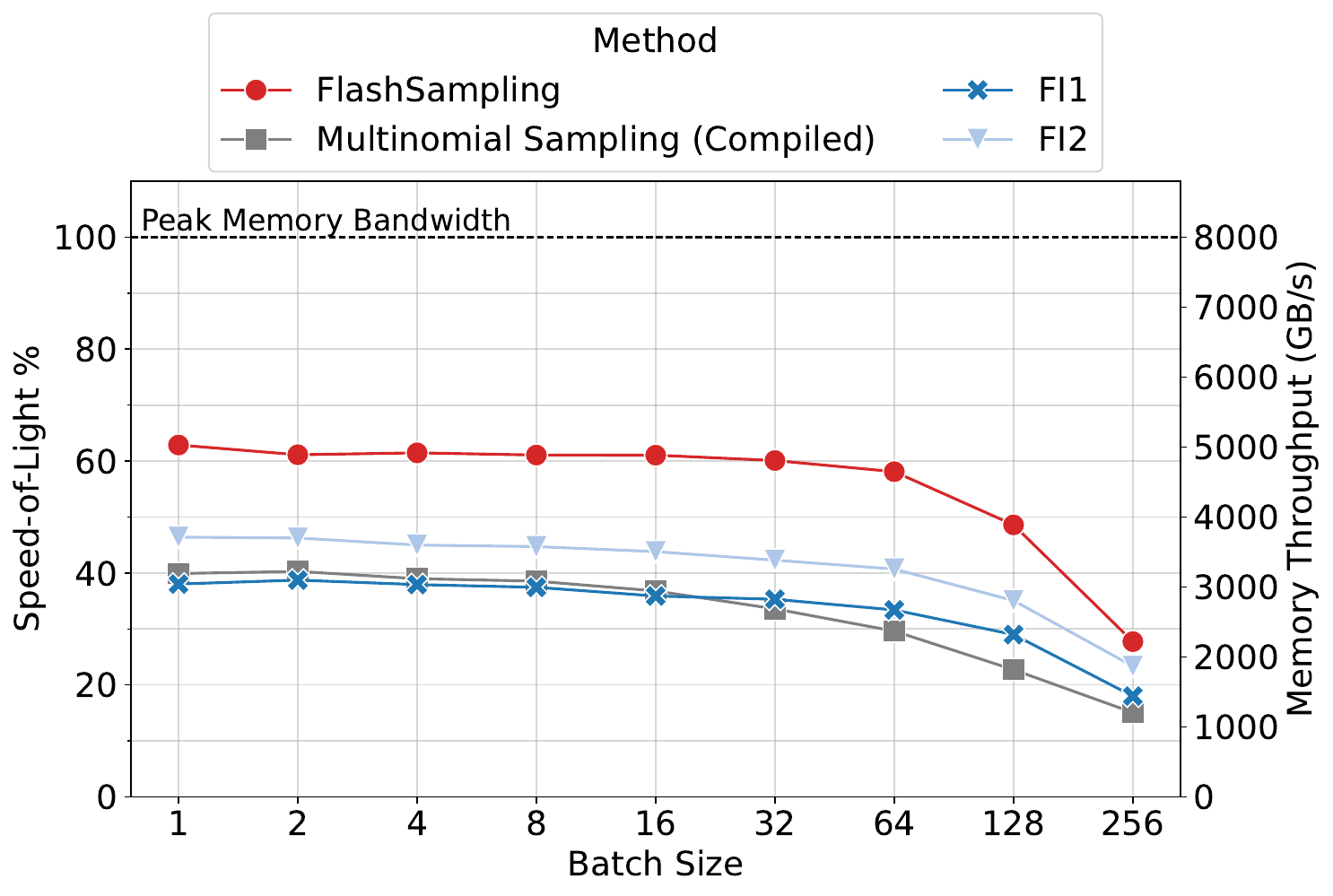}%
\caption{Roofline (left) and HBM bandwidth utilization (right) on B200 (higher is better).
Left: all methods track the memory-bound slope for $B \le 64$; FlashSampling sits above baselines because it avoids the logits round-trip.
Close to the ridge point ($\mathrm{AI} \approx 281$), performance flattens below the compute ceiling, where cuBLAS outperforms Triton.
Right: FlashSampling achieves higher bandwidth utilization than all baselines in the decode regime, confirming that fusion removes overhead rather than shifting it.}
\label{fig:roofline}
\end{figure}

\FloatBarrier

\IfFileExists{contents/appendix.tex}{\section{FlashSampling Algorithm Pseudocode}
\label{app:algorithms}

This appendix collects detailed pseudocode for the FlashSampling variants described in the main text.
\rev{All experiments in this paper use Algorithm~\ref{alg:fused}, on one GPU and under tensor parallelism.
The grouped, online, and distributed log-mass variants below are extensions and are not benchmarked.}

\paragraph{Streaming Gumbel-Max (standalone logits).}
Algorithm~\ref{alg:streaming} presents the basic one-pass streaming Gumbel-Max sampler over pre-materialized logits.

\begin{algorithm}[ht!]
\caption{Gumbel-Max sampling (standalone logits): streaming argmax over perturbed logits}
\label{alg:streaming}
\begin{algorithmic}[1]
\Require Logits $\vct{\ell}\in\R^V$, RNG state
\Ensure Sample index $i^\star\in\{1,\dots,V\}$
\State $m \gets -\infty$, $i^\star \gets 1$
\For{$i=1$ \textbf{to} $V$}
  \State $g \gets \textsc{Gumbel}(0,1)$ \Comment{via $g=-\log(-\log u)$, $u\sim \Unif(0,1)$}
  \State $s \gets \ell_i + g$
  \If{$s > m$}
    \State $m \gets s$, $i^\star \gets i$
  \EndIf
\EndFor
\State \Return $i^\star$
\end{algorithmic}
\end{algorithm}

\paragraph{Parallel Group-Gumbel-Max.}
Algorithm~\ref{alg:flash_parallel} extends streaming Gumbel-Max to a group-parallel setting where each group is processed by an independent threadblock.

\begin{algorithm}[ht!]
\caption{FlashSampling (parallel): Group-Gumbel-Max over groups}
\label{alg:flash_parallel}
\begin{algorithmic}[1]
\Require Input $\vct{x}\in\R^d$, weight matrix $\mat{W}\in\R^{d\times V}$, group size $g$ (so $V=mg$), RNG state
\Ensure Sample index $z\in\{1,\dots,V\}$ and optional log-normalizer $\ell_Z=\lse(\vct{y})$
\For{$k=0$ \textbf{to} $m-1$ \textbf{in parallel}}
  \State $\vct{y}_k \gets \mat{W}_k^\top \vct{x}\in\R^g$
  \State $z_k \gets \argmax_{j\in[g]} \big(y_{k,j} - \log(-\log u_{k,j})\big)$ \Comment{$u_{k,j}\!\sim\!\Unif(0,1)$}
  \State $L_k \gets \lse(\vct{y}_k)$
\EndFor
\State $k^\star \gets \argmax_{k\in[m]} \big(L_k - \log(-\log \bar u_k)\big)$ \Comment{$\bar u_k\!\sim\!\Unif(0,1)$}
\State $z \gets k^\star g + z_{k^\star}$ \Comment{map group-local index to global vocabulary index}
\State $\ell_Z \gets \lse([L_0,\dots,L_{m-1}])$ \Comment{optional}
\State \Return $(z,\ell_Z)$
\end{algorithmic}
\end{algorithm}

\paragraph{Sequential/online Group-Gumbel-Max.}
Algorithm~\ref{alg:flash_seq} provides a memory-efficient variant that streams groups one at a time.

\begin{algorithm}[ht!]
\caption{FlashSampling (sequential/online): streaming Group-Gumbel-Max with $O(g)$ working memory}
\label{alg:flash_seq}
\begin{algorithmic}[1]
\Require Input $\vct{x}\in\R^d$, weight matrix $\mat{W}\in\R^{d\times V}$, group size $g$ (so $V=mg$), RNG state
\Ensure Sample index $z\in\{1,\dots,V\}$ and optional log-normalizer $\ell_Z$
\Statex \textbf{Initialize with the first group.}
\State $\vct{y}_0 \gets \mat{W}_0^\top \vct{x}\in\R^g$
\State $L_0 \gets \lse(\vct{y}_0)$
\State $z_0 \gets \argmax_{j\in[g]} \big(y_{0,j} - \log(-\log u_{0,j})\big)$ \Comment{$u_{0,j}\!\sim\!\Unif(0,1)$}
\State $z \gets z_0$, $\ell \gets L_0$
\For{$k=1$ \textbf{to} $m-1$}
  \State $\vct{y}_k \gets \mat{W}_k^\top \vct{x}\in\R^g$
  \State $L_k \gets \lse(\vct{y}_k)$
  \State $\ell_{\text{new}} \gets \lse([\ell,\, L_k])$
  \State $p_{\text{replace}} \gets \exp(L_k - \ell_{\text{new}})$ \Comment{$=\frac{e^{L_k}}{e^\ell+e^{L_k}}$}
  \State Draw $u \sim \Unif(0,1)$
  \If{$u < p_{\text{replace}}$}
    \State $z_k \gets \argmax_{j\in[g]} \big(y_{k,j} - \log(-\log u_{k,j})\big)$ \Comment{sample within selected group}
    \State $z \gets k g + z_k$
  \EndIf
  \State $\ell \gets \ell_{\text{new}}$
\EndFor
\State $\ell_Z \gets \ell$ \Comment{optional}
\State \Return $(z,\ell_Z)$
\end{algorithmic}
\end{algorithm}

\paragraph{Distributed Group-Gumbel-Max.}
Algorithm~\ref{alg:flash_dist} extends FlashSampling to tensor-parallel vocabularies sharded across multiple GPUs.

\begin{algorithm}[ht!]
\caption{FlashSampling (distributed, tensor-parallel vocab): communicate $O(1)$ scalars per rank}
\label{alg:flash_dist}
\begin{algorithmic}[1]
\Require World size $n$. Rank $k\in\{0,\dots,n-1\}$ holds shard $\mat{W}^{(k)}\in\R^{d\times (V/n)}$ covering vocab indices $\{k\cdot V/n + 1,\dots,(k+1)\cdot V/n\}$. Input $\vct{x}\in\R^d$, RNG state.
\Ensure Global sample index $z\in\{1,\dots,V\}$ (and optional $\ell_Z$)
\State On each rank $k$:
\Statex \hspace{1em} compute local logits $\vct{y}^{(k)} \gets (\mat{W}^{(k)})^\top \vct{x}\in\R^{V/n}$
\Statex \hspace{1em} compute local log-mass $L_k \gets \lse(\vct{y}^{(k)})$
\Statex \hspace{1em} sample local index $\tilde z_k \sim \cat(\softmax(\vct{y}^{(k)}))$ \Comment{e.g., via Gumbel-Max / Group-Gumbel-Max / fused kernel}
\State All-gather $\{(L_k,\tilde z_k)\}_{k=0}^{n-1}$ to a coordinator (or perform an equivalent reduction)
\State Sample winning rank $k^\star \gets \argmax_{k\in[n]} \big(L_k - \log(-\log \bar u_k)\big)$ \Comment{$\bar u_k\!\sim\!\Unif(0,1)$}
\State $z \gets k^\star\cdot (V/n) + \tilde z_{k^\star}$ \Comment{convert rank-local index to global}
\State Optionally $\ell_Z \gets \lse([L_0,\dots,L_{n-1}])$
\State \Return $z$ (and $\ell_Z$)
\end{algorithmic}
\end{algorithm}

\section{Numerically Stable and Fast Gumbel Generation}
\label{app:gumbel_generation}

Gumbel noise can be generated as $g=-\log(-\log u)$ with $u\sim \Unif(0,1)$.
In GPU kernels, two issues matter:
\begin{itemize}
  \item \textbf{Numerical stability:} avoid $u=0$ or $u=1$ which lead to infinities.
  \item \textbf{Throughput:} the cost of generating random numbers and computing logs should not dominate.
\end{itemize}

\paragraph{Practical recipe.}
Given a 32-bit RNG output $r\in\{0,\dots,2^{32}-1\}$, map to
\[
u = \frac{r+1}{2^{32}+1} \in (0,1),
\]
then compute $g=-\log(-\log u)$. Many GPU RNG libraries (e.g.\ Philox, XORWOW) support generating floats in $(0,1)$ directly; the above mapping is a safe fallback.

\paragraph{Approximate log options.}
If exactness in the distribution is required, the Gumbel generation must be statistically correct. However, using fast approximate log implementations can introduce small distortions. \rev{All experiments in this paper use the standard $\log$.
We do not offer a fast-math mode: a kernel variant that uses the CUDA fast-math logarithm and division intrinsics was rejected by a chi-squared test with 10 million samples ($p=0.0004$), while the same kernel with the standard $\log$ passes ($p=0.90$).}

\section{Logits-Store Ablation}
\label{app:logits_ablation}

Table~\ref{tab:logits_ablation} reports the overhead of storing the computed $[B,V]$ logits tensor (FP32) back to HBM from within the fused kernel, compared with the predicted overhead of $2B/D$ from the cost model in Section~\ref{sec:cost_model}.
The ablation toggles a single flag in the kernel and changes nothing else.
The measured overhead was slightly larger than predicted, but tracked the trend closely.

\begin{table}[htbp]
\centering
\caption{Predicted vs.\ measured overhead of storing logits to HBM. Predicted: $2B/D$. Measured: relative slowdown of the fused kernel with the logits store enabled, averaged over 5 runs (B200 GPU).}
\label{tab:logits_ablation}
\begin{tabular}{l rr rr}
\toprule
& \multicolumn{2}{c}{$D{=}8192$, $V{=}128\text{k}$} & \multicolumn{2}{c}{$D{=}4096$, $V{=}152\text{k}$} \\
\cmidrule(lr){2-3}\cmidrule(lr){4-5}
$B$ & Predicted & Measured & Predicted & Measured \\
\midrule
1   & 0.02\% &  0.5\% & 0.05\% &  0.9\% \\
4   & 0.10\% &  0.7\% & 0.20\% &  1.2\% \\
16  & 0.39\% &  1.6\% & 0.78\% &  2.1\% \\
64  & 1.56\% &  3.6\% & 3.13\% &  4.8\% \\
128 & 3.13\% &  5.4\% & 6.25\% &  8.9\% \\
256 & 6.25\% &  8.1\% & 12.50\% & 15.2\% \\
\bottomrule
\end{tabular}
\end{table}

\section{Returning Log-Normalizers or Max Values}
\label{app:log_normalizers}

Some applications need $\log Z = \log\sum_j e^{\tilde{\ell}_j}$, for example to compute log-probabilities.
The core FlashSampling sampler does not need $\log Z$, but it can be added as an optional mode by accumulating a numerically stable log-sum-exp alongside sampling.
In fused settings, this requires extra work in the epilogue, so we treat it as an optional feature rather than part of the core design.
}{}

\section{Licenses of Existing Assets}
\label{app:licenses}
Table~\ref{tab:licenses} lists the third-party software, models, and datasets used in this paper, with citation and license.
All assets are used in accordance with their respective licenses.

\begin{table}[ht!]
\centering
\caption{Licenses of existing assets used in this paper.}
\label{tab:licenses}
\begin{tabular}{lll}
\toprule
Asset & Reference & License \\
\midrule
PyTorch & \citet{paszke2019pytorch} & \href{https://github.com/pytorch/pytorch/blob/main/LICENSE}{BSD-3-Clause} \\
Triton & \citet{tillet2019triton} & \href{https://github.com/triton-lang/triton/blob/main/LICENSE}{MIT} \\
vLLM & \citet{kwon2023efficient} & \href{https://github.com/vllm-project/vllm/blob/main/LICENSE}{Apache-2.0} \\
FlashInfer & \citet{ye2025flashinfer} & \href{https://github.com/flashinfer-ai/flashinfer/blob/main/LICENSE}{Apache-2.0} \\
Qwen3 (1.7B, 8B, 32B) & \citet{yang2025qwen3} & \href{https://huggingface.co/Qwen/Qwen3-8B/blob/main/LICENSE}{Apache-2.0} \\
Llama-3.3-70B & \citet{grattafiori2024llama} & \href{https://www.llama.com/llama3_3/license/}{Llama 3.3 Community License} \\
GSM8K & \citet{cobbe2021training} & \href{https://github.com/openai/grade-school-math/blob/master/LICENSE}{MIT} \\
\bottomrule
\end{tabular}
\end{table}

\section{Fused Top-\texorpdfstring{$k$}{k} and Top-\texorpdfstring{$p$}{p} Sampling}
\label{app:topk}

This appendix describes a fused implementation of top-$k$ followed by top-$p$, the order that vLLM and FlashInfer apply by default (Appendix~\ref{sec:truncation}), and reports its latency.

\paragraph{Algorithm.}
Stage 1 is the fused GEMM of Algorithm~\ref{alg:fused} with a different epilogue.
For every row it writes the largest logit of every 8 consecutive vocabulary entries and of every 128.
Nothing is sorted and no index is stored.
Stage 2 finds, by bisection over these maxima, a threshold $\tau$ that between $k+2$ and $k+4$ of the 8-entry groups reach.
It recomputes the logits of those groups (at most 256 dot products per row), sorts that short list, keeps the $k$ largest, applies top-$p$, and draws one sample with Gumbel-Max.
By the partition argument of Appendix~\ref{sec:truncation}, every one of the $k$ largest logits lies in a group whose maximum is at least $\tau$, so the kept set is the exact top-$k$ set.
The $[B,V]$ logits tensor is never written.
$k$, $p$, and the temperature are read per row inside the kernels, so requests in one batch may use different values.

\paragraph{Scope.}
The implementation covers $k \le 28$ on a single GPU, which includes the Qwen3 default of $k=20$.
Requests with top-$p$ alone or min-$p$ take the standard path.

\paragraph{Kernel latency.}
Table~\ref{tab:topk_kernel} compares the fused path with FlashInfer and vLLM at $V=151{,}936$, $D=4{,}096$, temperature $0.6$, top-$k$=20, and top-$p$=0.95 (the Qwen3 defaults).
The run used one NVIDIA H200.
With truncation on, the fused path takes 24.7\% to 36.0\% less time than FI1 and 41.5\% to 52.7\% less than vLLM's default path.
At $B \le 64$, top-$k$ and top-$p$ add 8.0\% to 11.7\% to the latency of the temperature-only kernel on the same GPU.
Turning on the same truncation adds 12.5\% to 21.4\% to FI1's latency.
The output has the same top-$k$ set and the same top-$p$ support as a reference computed from the full logits in every row at every batch size.
The buffers that Stage 1 writes are 13.5\% of the size of the FP32 logits tensor.

\begin{table}[ht]
\centering
\caption{Kernel latency in ms with top-$k$=20 and top-$p$=0.95 on one H200 ($V=151{,}936$, $D=4{,}096$, temperature $0.6$).
Medians over five interleaved rounds.}
\label{tab:topk_kernel}
\begin{tabular}{l rrrr}
\toprule
Method & $B=1$ & $B=16$ & $B=64$ & $B=256$ \\
\midrule
FlashSampling, fused top-$k$ then top-$p$ & \textbf{0.331} & \textbf{0.337} & \textbf{0.380} & \textbf{0.723} \\
FI1, same $k$ and $p$ & 0.440 & 0.476 & 0.593 & 1.114 \\
vLLM default filter and sampler & 0.566 & 0.638 & 0.743 & 1.528 \\
PyTorch sort, then Multinomial & 0.605 & 1.072 & 1.900 & 5.418 \\
\bottomrule
\end{tabular}
\end{table}

\paragraph{End to end.}
Table~\ref{tab:topk_tpot} reports TPOT for Qwen3-1.7B served by vLLM at the model's default sampling parameters (temperature $0.6$, top-$k$=20, top-$p$=0.95) on one H200, with the protocol of Section~\ref{sec:vllm}.
The reduction is the median over five paired runs; its standard deviation is at most 0.4 points up to $B=32$ and 1.6 points at $B=64$.

\begin{table}[ht]
\centering
\caption{Median TPOT in ms for Qwen3-1.7B with top-$k$=20 and top-$p$=0.95 on one H200, and the TPOT reduction.}
\label{tab:topk_tpot}
\begin{tabular}{l rrrrrrr}
\toprule
$B$ & 1 & 2 & 4 & 8 & 16 & 32 & 64 \\
\midrule
vLLM & 2.289 & 2.462 & 2.598 & 2.609 & 2.558 & 2.742 & 3.167 \\
FlashSampling & 2.059 & 2.073 & 2.113 & 2.140 & 2.234 & 2.413 & 2.737 \\
Reduction & 10.1\,\% & 15.9\,\% & \textbf{18.7}\,\% & 17.8\,\% & 12.7\,\% & 11.8\,\% & 13.4\,\% \\
\bottomrule
\end{tabular}
\end{table}

\end{document}